\documentclass[journal, 10pt, a4paper,twocolumn]{IEEEtran}

\usepackage[cmex10]{amsmath}
\usepackage{amsthm, amsfonts, amssymb, bm}
\usepackage{array}
\usepackage{color}
\usepackage{epsfig}
\usepackage{algorithm, algorithmic}
\usepackage[active]{srcltx}
\usepackage{subfig}
\usepackage{hyperref}
\usepackage{rotating}
\newtheorem{theorem}{Theorem}[section]
\newtheorem{lemma}[theorem]{Lemma}
\newtheorem{corollary}[theorem]{Corollary}
\newtheorem{example}[theorem]{Example}

\newtheorem{proposition}[theorem]{Proposition}

\newtheorem{definition}[theorem]{Definition}

\def\norm#1{\|#1\|}

\def\R{{\mathbb R}}

\def\be{\begin{eqnarray*}}
\def\ee{\end{eqnarray*}}
\def\beq{\begin{equation}}
\def\eeq{\end{equation}}

\def\2q{\quad\quad}

\def\R{{\mathbb R}}

\def\:{{\,:\,}}

\def\norm#1{{\left\|\,#1\,\right\|}}
\def\abs#1{{\left|\,#1\,\right|}}

\def\norm #1{\|#1\|}
\def\abs #1{|#1|}

\def\argmin{\mathop{\rm arg\,min}}

\def\inprod#1#2{\langle #1,\,#2\rangle}

\def\itb{\begin{itemize}}
\def\ite{\end{itemize}}
\def\bit{\begin{itemize}}
\def\eit{\end{itemize}}
\def\dom{\hbox{dom}}

\begin{document}

\title{Logitron: Perceptron-augmented classification model based on an extended logistic loss function}

\author{Hyenkyun Woo
\thanks{School of Liberal Arts, Korea University of Technology and Education, \; hyenkyun@koreatech.ac.kr, hyenkyun@gmail.com}}

\markboth{...,~Vol.~x, No.~x,\;\; \today}{}

\maketitle

\begin{abstract}
Classification is the most important process in data analysis. However, due to the inherent non-convex and non-smooth structure of the zero-one loss function of the classification model, various convex surrogate loss functions such as hinge loss, squared hinge loss, logistic loss, and exponential loss are introduced. These loss functions have been used for decades in diverse classification models, such as SVM (support vector machine) with hinge loss, logistic regression with logistic loss, and Adaboost with exponential loss and so on. In this work, we present a Perceptron-augmented convex classification framework, {\it Logitron}. The loss function of it is a smoothly stitched function of the extended logistic loss with the famous Perceptron loss function. The extended logistic loss function is a parameterized function established based on the extended logarithmic function and the extended exponential function. The main advantage of the proposed Logitron classification model is that it shows the connection between SVM and logistic regression via polynomial parameterization of the loss function. In more details, depending on the choice of parameters, we have the Hinge-Logitron which has the generalized $k$-th order hinge-loss with an additional $k$-th root stabilization function and the Logistic-Logitron which has a logistic-like loss function with relatively large $\abs{k}$. Interestingly, even $k=-1$, Hinge-Logitron satisfies the classification-calibration condition and shows reasonable classification performance with low computational cost. The numerical experiment in the linear classifier framework demonstrates that Hinge-Logitron with $k=4$ (the fourth-order SVM with the fourth root stabilization function) outperforms logistic regression, SVM, and other Logitron models in terms of classification accuracy.
\end{abstract}

\begin{IEEEkeywords}
Extended exponential function, extended logarithmic function, logistic regression, extended logistic regression, sigmoid, extended sigmoid function, hinge loss, higher-order hinge loss, support vector machine, Perceptron
\end{IEEEkeywords}


\section{Introduction}
Learning a decision boundary for the classification of data observed in a real world is a fundamental and important process in machine learning~\cite{murphy12,shalev14} and thus various classification models are introduced during the last several decades; for instance, logistic regression~\cite{cox58}, SVM (support vector machine)~\cite{vapnik00}, decision trees~\cite{breiman84}, random forests~\cite{breiman01}, neural networks~\cite{rosenblatt57,bengio15}, and boosting~\cite{freund97,friedman00,chen16} have been developed. Among these diverse classification models, logistic regression is a probability-based popular model~\cite{thomas17}. In this work, we are mainly interested in a convex classification model {\it Logitron} built up with the classic Perceptron loss function and the extended logistic loss function, which is not a specific loss function but a polynomial parameterized loss function based on the extended logarithmic function~\cite{woo17} and the extended exponential function~\cite{woo18}. Note that the extended logistic loss function includes a lot of surrogate loss functions appearing in various margin-based classification models. For instance, unhinge loss~\cite{rooyen15}, exponential loss~\cite{freund97}, logistic loss~\cite{cox58,friedman00}, sigmoid function~\cite{mason00} and its variant Savage loss~\cite{masnadi08}, and so on. Among them, the non-convex loss functions or unbounded convex loss function, e.g., sigmoid, Savage loss, and unhinge loss, are mainly used for robust boosting classification model. Last but not least, \cite{ding10} has introduced $t$-logistic regression based on the $t$-exponential family for robustness of the classification model. 

Let us start with the standard binary classification model~\cite{lecun06,shalev14,vapnik00}.  A formal binary classifier $g_f(x)$ is simply defined as $g_f(x) = sign(f(x))$ where $sign(f(x))=+1$ if $f(x)>0$ and $-1$ otherwise. Here $f(x) : {\cal X} \rightarrow \R$ is a predictor (or score function) and ${\cal X} = \{ x \in \R^n \;|\; \norm{x}_{\infty} \le R_{\cal X} \}$ is a feature space and $R_{\cal X}$ is a constant. Note that $f \in {\cal F}$ where ${\cal F}$ is a function space defined based on a category of classification models. For instance, when we learn a hyper-plane of the feature space, we set ${\cal F} = \{ \inprod{w}{x} + b \;|\; w \in {\cal W}(n), b \in \R , x \in {\cal X} \}$ with ${\cal W}(n) = \{ w \in \R^n \;|\; \norm{w}_{\infty} < R_{\cal W} \}$ and $R_{\cal W}$ is a constant. For more advanced classification models such as ensemble learning models and (deep) neural networks, a sophisticated function space is required. For ensemble learning models~\cite{shalev14}, i.e., boosting and bagging, ${\cal F} = \{  \inprod{w}{g(x)} \;|\; w\in {\cal W}(N), g = (g_1,...,g_N) \in {\cal B}^N  \}$ where ${\cal B}$ is a function space of so-called base (or weak) classifiers. For (deep) neural network~\cite{zhang15,glorot11,elsayed18,collobert04,bengio15}, ${\cal F} = {\cal N}_r$ with ${\cal N}_r = \{ \inprod{w}{\sigma(f(x))} \;|\; f \in {\cal N}^N_{r-1}, w\in {\cal W}(N) \}$ and ${\cal N}_1 = \{\inprod{w}{x} \;|\; w \in {\cal W}(N) \}$. Here $\sigma(x)$, which is known as an activation function, is the only nonlinear function in neural network. A typical example is the sigmoid function~\cite{collobert04}. Recently, $\max$ function-based rectified linear unit (i.e., ReLU) is used as an activation function for deep neural network~\cite{glorot11}. For kernel-based learning model, which is a straightforward extension of the linear classifier, we can set ${\cal F} = \{ \sum_{i=1}^N w_ik(x_i,x) \;|\; \inprod{w}{Kw} \in [0,R^2_{\cal W}] , K = [k(x_i,x_j)] \in \R^{N \times N},\; x_i \hbox{ is an observed data}  \}$. For more details on various classification model and the corresponding function space, see \cite{murphy12,lecun06,zhang15,boucheron05} and references therein. Unless otherwise stated, in this work, we assume that ${\cal F}$ is a linear function space
\begin{equation}\label{funcSpace}
{\cal F} = \{ \inprod{w}{x} + b \;|\; w \in {\cal W}(n), b \in \R, \hbox{ and } x \in {\cal X}  \}
\end{equation}

Now, the question is that, from the collected training data $(x_1,y_1),..., (x_N,y_N) \in {\cal X} \times {\cal Y}$ with ${\cal Y} = \{-1,+1\}$, how can we find the right prediction function $f$ minimizing $\hbox{Prob}(g_f(x) \not= y)$?  A simple approach is to directly minimize the misclassification error (i.e., the zero-one loss function~\cite{nguyen13}), $\hbox{Prob}(g_f(x) \not= y) = \frac{1}{N}\sum_{i=1}^N \ell_{0/1}(y_if(x_i))$ where $\ell_{0/1}(z) = {\bf 1}(-z)$ and ${\bf 1}(\cdot)$ is an indicator function, i.e., ${\bf 1}(a) = 1$ if $a > 0$ and $0$ otherwise. Although the zero-one loss function $\ell_{0/1}$ is simple and easy to understand, it is non-differentiable and non-convex. Finding global optimums of it is a typical NP-hard problems~\cite{nguyen13}. Instead of using bilevel zero-one loss function, we can consider convex relaxations of that. For instance, we have the classic Perceptron loss function $\ell_P(z) = \max(0,-z)$ and the corresponding minimization problem (i.e. Perceptron~\cite{rosenblatt57}):
\begin{equation}\label{maxloss}
\min_{f \in {\cal F}}\; \sum_{i=1}^N \ell_P(y_if(x_i))
\end{equation}
where $\ell_P(yf(x))=\abs{f(x)}$ is linearly penalized with respect to $f(x)$ only if $g_f(x) \not=y$. Actually, it is easy to find a solution of the Perceptron model~\eqref{maxloss} with the subgradient-based method, known as the Perceptron algorithm. The main concern of \eqref{maxloss} is that it is sensitive to the noise (or data) near the decision boundary, i.e., $(x,y) \in D(\varepsilon) = \{ (x,y) \in {\cal X} \times {\cal Y} \;|\;  \abs{yf(x)} \le \varepsilon \}.$ In fact, \eqref{maxloss} does not have sufficient margin. As a solution of the insufficiency of margin, we can consider higher-order SVM~\cite{vapnik00,bartlett06,janocha17}:
\begin{equation}\label{hingeloss}
\min_{f \in {\cal F}}\; \sum_{i=1}^N \ell_{H,k}(y_if(x_i))
\end{equation}
where $k \in {\mathbb N}$ and $\ell_{H,k}(y_if(x_i)) = (\max(0, 1- y_if(x_i)))^{k}$ is the higher-order hinge-loss function. Especially, when $k=1$, \eqref{hingeloss} is the classic SVM, known as the max-margin classifier, with the first-order hinge loss function~\cite{vapnik00} and when $k=2$, it is known as L2SVM (or squared SVM)~\cite{fan08}. Recently, the third order hinge loss function $\ell_{H,3}(z)$ is introduced as an activation function for the deep neural network~\cite{janocha17}. To the best of author's knowledge, $k$-th hinge loss function $\ell_{H,k}(z)$ with $k \ge 4$ is not introduced in literatures. In this work, we study stabilized $k$-th order SVM which has arbitrary $k \in \mathbb{N}$ within the proposed Logitron framework.


As observed in \cite{murphy12}, the misclassification error $\hbox{Prob}(g_f(x) \not= y)$ can also be formulated with the sigmoid probability function $p_f(x) = \frac{1}{1 + \exp(-f(x))}$ and the corresponding classifier $g_f(x) = sign(p_f(x)-0.5)$. In fact, by using the negative log-likelihood of the Bernoulli distribution which has the sigmoid function $p_f(x)$ as the probability density function, we get the famous logistic loss function $\ell_L(yf(x)) = \log(1 + \exp(-yf(x)))$ and the corresponding logistic regression formulation:
\begin{equation}\label{clogistic}
\min_{f \in {\cal F}}\; \sum_{i=1}^N\; \ell_L(y_if(x_i))
\end{equation}
where $(x_i,y_i) \in {\cal X} \times \{-1,+1\}$. The main advantage of this model is that the logistic loss function $\ell_L$ is sufficiently smooth and the gradient of it is the sigmoid probability function. That is, let $y=-1$ then we have
$\frac{d\ell_L(-f)}{df} = \frac{1}{1 + \exp(-f(x))} = p_{f}(x)$. Though, the logistic regression is a typical example of the margin-based classification model, since $\ell_L(z) > 0$ for all $z \in \R$, it is unclear how to connect this model to the SVM, the max-margin classifier. 

The proposed Logitron, having the Perceptron-augmented extended convex logistic loss function, is inherently similar to the logistic regression with an additional margin control parameter. Roughly, we can say that the Logitron is the generalized $q$-th order SVM with an additional stabilization $q$-th root function ($q \in \R \setminus [0,1)$). Depending on the choice of parameters, we have the Hinge-Logitron with hinge-like loss function with relatively small value of $\abs{q}$ and the Logistic-Logitron with logistic-like loss function with relatively larger value of $\abs{q}$. In terms of logistic regression framework, when $\abs{q}$ is relatively large, the generalized $q$-th order SVM corresponds to the exponential function and the stabilization $q$-th root function corresponds to the logarithmic function. Interestingly, even $q<0$, we have classification model which satisfying the classification-calibrated condition~\cite{bartlett06}. In fact, when $q=-1$, the Hinge-Logitron is implementable with simple elementary mathematical operations such as division and show reasonable classification performance. Note that the margin of the Logitron loss function is defined as the intersection point of the closure of the domain of the extended exponential function and the Perceptron loss function. When the intersection point is located on the positive real line ($q>0$), it corresponds to the classic margin. Interestingly, the Logitron loss function is sufficiently smooth on its entire domain $\R$ under the mild restriction of the parameter and therefore, we can easily use the conventional gradient-based optimization model to find a solution of the Logitron model. 

As regards the numerical experiments, for multi-class classification problem, we have used OVA (one-vs-all) framework. The Hinge-Logitron {\bf H-4} (i.e., the fourth-order SVM with the fourth-root stabilization function) shows the best performance in learning hyperplanes~\eqref{funcSpace}. Compared to the conventional second-order SVM, known as L2SVM~\cite{fan08}, the proposed Hinge-Logitron {\bf H-2} (i.e., second order SVM with root stabilizer function) shows better performance in terms of classification accuracy. The Logistic-Logitron {\bf L-} (i.e., a group of the Logitron model with $q=5,6,8,12$) shows the best performance with respect to the Friedman ranking~\cite{delgado14}. As a by-product of the generalization to the negative region of $q$, we obtain classification-calibrated new classification model. This new classification model also shows better performance than the conventional logistic regression and SVM in terms of the classification accuracy.
\subsection{Notation\label{sec1-1}}
We briefly review a convex function and related useful notations such as extended-valued function. See \cite{roc70,hir96,bauschke11} for more details.

Let $h : \hbox{dom}(h) \rightarrow \R$ be a convex, lower semicontinuous, and proper function on its convex domain 
\begin{equation}\label{dom}
\hbox{dom} (h) = \{ z \in \R \;|\; h(z) \not= \emptyset \}.
\end{equation}
As observed in \cite{hir96}, the convexity of $h$ can be extended to the whole real line $\R$ by using the extended-valued function $h^e  : \R \rightarrow \R_{\infty}$:
\begin{equation}\label{extendedfunc}
h^e(z) = \left\{\begin{array}{l} h(z) \quad z \in \Omega \\ +\infty \quad z \not\in \Omega  \end{array}\right.
\end{equation}
where $\R_{\infty} = \R \cup \{ +\infty \}$ and $\Omega = \hbox{dom}(h)$. Depending on applications~\cite{woo17}, $\Omega$ can be any convex set in $\R$. Unless otherwise stated, as suggested in \cite{hir96}, a convex function in this work is an extended-valued convex function~\eqref{extendedfunc} and, for simplicity, we will drop the superscript 'e' in the extended-valued function $h^e$. In $\R_{\infty}$ (an extended-valued real number system), we introduce several arithmetic operations with $+\infty$ which are useful later. That is, $a+\infty = \infty$ for all $a \in \R$, $1/\infty = 0$ (it means $\lim_{n \rightarrow +\infty}1/n = 0$), $1/0 = +\infty$  (it means $\lim_{\epsilon \rightarrow 0_+} 1/\epsilon = +\infty$), and $1=\frac{\infty + a}{\infty +b}$ (it means $\lim_{z \rightarrow c_{\alpha}} \frac{h(z)+a}{h(z)+b}=1$ with $h(c_{\alpha})=+\infty$ and $a,b \in \R$). 

Let $\Omega$ be any convex set in $\R$. Then $int(\Omega)$ is the interior of $\Omega$ and $bd(\Omega) = cl(\Omega) \setminus int(\Omega)$ is the boundary of $\Omega$. Here $cl(\Omega)$ is the closure of $\Omega$. We also set $\R_{++} = \{ z \in \R \;|\; z >0 \}$, $\R_{+} = \{ z \in \R \;|\; z \ge 0 \}$, $\R_{\ge c_{\alpha}} = \{ z \in \R \;|\; z \ge c_{\alpha} \}$, and $\R_{> c_{\alpha}} = \{ z \in \R \;|\; z > c_{\alpha} \}$. The corresponding negative intervals are also defined in the same way. Note that $\mathbb{Q}$ is a set of rational number, $\mathbb{Z}$ is a set of integer, and $\mathbb{N}$ is a set of natural number. Additionally, $\hbox{dom}(h)$ is always assumed to be a convex set, irrespective of convexity of $h$. 
 
\subsection{Overview}
The paper is organized as follows. In Section \ref{sec2}, we review extended exponential and logarithmic functions which are studied in~\cite{woo17,woo18}. In Section \ref{sec3}, we introduce the extended logistic loss function defined with the extended exp and log function and the corresponding general classification framework, {\it Logitron}. The loss function of it is a smoothed stitching of the Perceptron loss and the restricted version of the extended logistic loss. In Section \ref{sec4}, we reinterpret the Logitron by the generalized $q$-th order SVM with the $q$-th root stabilization function. Here $q \in \mathbb{Q} \cap \R \setminus [0,1)$. Actually, L2SVM, known as the SVM with squared hinge loss, can be reformulated into the Hinge-Logitron {\bf H-2} with an additional root stabilization function.  In Section \ref{sec5}, we evaluate the performance of the proposed Logitron with more than one hundred datasets~\cite{delgado14}. The conclusions are given in Section \ref{sec6}.

\section{Extended exponential function and extended logarithmic function  \label{sec2}}
In this Section, we review the extended exponential function~\cite{woo18} and its inverse function, the extended logarithmic function~\cite{woo17}. These extended elementary functions are fundamental ingredients of the extended logistic loss function and the Logitron classification model. 

Firstly, let us start with the definition of an extended logarithmic function~\cite{woo17}. It is a generalized logarithmic function~\cite{amari16,amari00,tsallis09} with an additional scaling parameter. Later, we will explain the role of an additional parameter in details in terms of the margin of the Logitron classification model.  
\begin{definition}\label{elogf}
Let $\alpha \in \R_{+},$ $u \in \hbox{dom} (\ln_{\alpha,c}) \subseteq \R.$ Then the extended logarithmic function is defined as
\begin{equation}\label{genlog}
\ln_{\alpha,c}(u) = \int_{c}^{u} x^{-\alpha} dx
\end{equation}
where $c \in \R_{c,u} = \{c \in \R \setminus \{ 0 \} \;|\; \ln_{\alpha, c}(u) \in \R \;\hbox{ and }\; \hbox{sign}(u) = \hbox{sign}(c) \}$. 
After integration, we have a simplified version of it by
\begin{equation}\label{exlog}
\ln_{\alpha,c}(u) = \left\{\begin{array}{l} \ln\left(\frac{u}{c}\right),\hskip 2.3cm \hbox{ if } \alpha=1 \\ \frac{1}{1-\alpha}(u^{1-\alpha} - c^{1-\alpha}), \quad \hbox{ otherwise } \end{array}\right.
\end{equation}
\end{definition}
The convexity of $\ln_{\alpha,c}$ depends on parameters $\alpha$ and $c$. See \cite{woo17}, for more detail characterization of the domain of $\ln_{\alpha,c}$. In fact, $\hbox{dom}(\ln_{\alpha,c})$ is rather complicated. As observed in \cite{woo17}, the domain of $\ln_{\alpha,c}$ should be determined to meet the requirement of applications, such as $\beta$-divergence~\cite{woo17} and statistical Tweedie distribution~\cite{woo18}. If we set $c=1$, the extended log function $\ln_{\alpha,1}(u) = \int_1^u x^{-\alpha}dx$ becomes the generalized log function~\cite{amari16,tsallis09}. 

Secondly, we introduce an extended exponential function~\cite{woo18}, the scaled version of  the generalized exponential function~\cite{amari16,tsallis09,ding10}.  Note that the scaling parameter $c$ of the extended exponential function is very important in the Logitron loss function, since it controls the margin of the classification model unlike the generalized exp function.  

\begin{definition}\label{eexpf}
Let $\alpha \in \R_+,$ $v \in \hbox{dom} (\exp_{\alpha,c}) \subseteq \R,$ and
\begin{equation}\label{genexp}
\exp_{\alpha,c}(v) = y
\end{equation}
where $\exp_{\alpha,c}(v)$ is defined to satisfy the following relation:
\begin{equation}\label{genexp2}
v = \int_{c}^{y} x^{-\alpha} dx
\end{equation}
where $c \in \R_{c,y} = \{c \in \R \setminus \{ 0 \} \;|\; sign(y) = sign(c) \}.$ After integration, we get a simplified version of it by
\begin{equation}\label{exexp}
\exp_{\alpha,c}(v) =
\left\{\begin{array}{l} 
c\exp(v), \hskip +2.9cm \hbox{ if } \alpha=1\\
(c^{1-\alpha} + (1-\alpha)v)^{1/(1-\alpha)}, \quad \hbox{ otherwise }
\end{array}\right.
\end{equation}
\end{definition}
 If we set $c=1$, the extended exponential function, $\exp_{\alpha,1}(v)$ becomes generalized exponential function in \cite{amari16,ding10}. The convexity of the extended exp function $\exp_{\alpha,c}$ depends on parameters $\alpha, c$ and thus  the structure of $\hbox{dom}(\exp_{\alpha,c})$ is complicated~\cite{woo18}. What is even worse, the extended exponential function defined in Definition \ref{eexpf} does not have inverse relation with the extended log function defined in Definition \ref{elogf}. 
 Additionally, as observed in \cite{woo17,woo18}, the domains of them should be carefully selected to meet various conditions related to the high level structures. A typical example is a condition of convex function of Legendre type~\cite{roc70}. With the restricted domains satisfying the condition of the convex function of Legendre type, it is possible to obtain rather complicated dual relation between $\beta$-divergence and the Tweedie distribution~\cite{woo18,woo17, jorgensen97, banerjee05}. 

In this work, we are going to use extended exp and log functions for classification purpose only.  Hence, we significantly reduce domains of them. See Table \ref{table7} for more details.
\begin{table}[h]
\centerline{\begin{tabular}{c|cclc|ccc}
\hline\hline
   & $\alpha = 1$  & $0 \le \alpha < 1$  & $\alpha>1$  \\ \hline 
 $\dom (\ln_{\alpha,c})$ & $\R_{++}$  & $\R_+$  & $\R_{++}$ \\ \hline
$\dom (\exp_{\alpha,c})$   & $\R$  & $\R_{\ge c_{\alpha}}$ & $\R_{< c_{\alpha}}$ \\ \hline\hline
\end{tabular}}
\caption{The restricted convex domains of the extended exponential function $\exp_{\alpha,c}(x)$ and the extended logarithmic function $\ln_{\alpha,c}(x)$ for classification purpose. Note that the domain of $\exp_{\alpha,c}$ (i.e., the range of $\ln_{\alpha,c}$) can be adjusted by controlling the auxiliary parameter $c$. Here $c >0$ and $c_{\alpha} = \frac{1}{\alpha-1} c^{1-\alpha}$.}\label{table7}
\end{table}
Now, $\exp_{\alpha,c}$ and $-\ln_{\alpha,c}$ with domains in Table \ref{table7} are convex and extended-valued functions. We summarize various properties of them below.
\begin{proposition}\label{propexp}
Let $(\alpha,c) \in (\R_+ \setminus\{ 1 \}) \times \R_{++}$. Then the extended exp function $\exp_{\alpha,c} : \R \rightarrow \R_{\infty}$ and the extended log function $-\ln_{\alpha,c} : \R \rightarrow \R_{\infty}$ have the following properties with the domains in Table~\ref{table7}. Here $x_{\alpha} = \frac{x^{1-\alpha}}{\alpha-1}$ and $0_+ = \lim_{\varepsilon \rightarrow 0^+} 0 + \varepsilon$.
\begin{enumerate}
\item $\forall \alpha \in (0,1)$, $\exp_{\alpha,c}$ is strictly increasing and, $\forall \alpha \in \R_+ \setminus \{ 1 \}$, $\ln_{\alpha,c}$ is strictly increasing.
\item $\exp_{\alpha,c}$ and $-\ln_{\alpha,c}$ are convex functions on their domains.
\end{enumerate}
\end{proposition}
\begin{proof}
\begin{itemize}
\item[1)] Under domains in Table \ref{table7}, it is easy to see $\ln_{\alpha,c}(x) = c_{\alpha}-x_{\alpha}$ is strictly increasing. In case of $\exp_{\alpha,c}(x)$, since $1- \frac{x}{c_{\alpha}}>0$ for all $x \in \dom(\exp_{\alpha,c})$, $\exp_{\alpha,c}(x) = c (1-\frac{x}{c_{\alpha}})^{\frac{1}{(1-\alpha)}}$ is strictly increasing when $\alpha \in (0,1)$.
\item[2)] Since $x_{\alpha}$ is convex for all $x \in \dom(x_{\alpha}) \cap \R_+$, $-\ln_{\alpha,c}(x) = x_{\alpha} - c_{\alpha}$ is a convex function on its domain in Table \ref{table7}. For all $x \in int(\dom(\exp_{\alpha,c}))$, we have $\exp''_{\alpha,c}(x) > 0$ and convexity can be easily extended to the boundary of the domain in Table \ref{table7}.
\end{itemize}
\end{proof}
Now, we will show that the extended exponential function~\eqref{exexp} and the extended logarithmic function~\eqref{exlog} are well-defined (i.e., $\exp_{\alpha,c} = \ln_{\alpha,c}^{-1}$ and $\ln_{\alpha,c} = \exp_{\alpha,c}^{-1}$). Actually, we show the isomorphic inverse relation between \eqref{exlog} and \eqref{exexp} below under the restricted domains in Table \ref{table7}. 
\begin{lemma}\label{lemmaG}.
Let $\alpha \in \R_+$ and $c \in \R_{++}$. Then we have the bijective mapping between the the extended log~\eqref{exlog} and the extended exp~\eqref{exexp} functions with the restricted domains in Table \ref{table7}: 
$$
\ln_{\alpha,c} : \dom (\ln_{\alpha,c}) \rightarrow \dom (\exp_{\alpha,c})
$$
with the corresponding inverse map $\ln_{\alpha,c}^{-1} = \exp_{\alpha,c}$
$$
\exp_{\alpha,c}: \dom (\exp_{\alpha,c}) \rightarrow \dom (\ln_{\alpha,c})
$$
\end{lemma}
Note that the proof of Lemma \ref{lemmaG} can be easily derived from Table \ref{table7} and the definition of the extended exp~\eqref{exexp} and extended log~\eqref{exlog}. The following Lemma is useful while we define the loss function for the classification model. In fact, the range of the extended exponential function always equals to the domain of the extended logarithmic function, irrespective of choice of parameters $\alpha$ and $c$.
\begin{lemma}\label{lemmax}
For any $\alpha, \beta \in \R_+$, $c_1,c_2 \in \R_{++}$, and domains in Table \ref{table7}, we have
\begin{equation}\label{eqx}
\hbox{ran}(\exp_{\beta,c_{2}}) = \dom(\ln_{\alpha,c_{1}})
\end{equation}
\end{lemma}
\begin{proof}
Due to the isomorphic mapping in Lemma \ref{lemmaG} (i.e., $\ln^{-1}_{\alpha,c} = \exp_{\alpha,c}$ and $\ln_{\alpha,c} = \exp^{-1}_{\alpha,c}$) on domains defined in Table \ref{table7}, we have
$$
\hbox{ran}(\exp_{\alpha,c}) = \dom(\ln_{\alpha,c})
$$  
As observed in Table \ref{table7}, the domain of $\ln_{\alpha,c}$ does not depend on the choice of $\alpha$ and $c$. Hence, we have
$$
\hbox{ran}(\exp_{\beta,c_{2}}) = \dom(\ln_{\alpha,c_{1}})
$$
for any choice of $\alpha, \beta \in \R$ and $c_1,c_2 \in \R_{++}$.
\end{proof}
The independency of the parameter $\alpha$ and $\beta$ introduced in Lemma \ref{lemmax} is very useful while we characterize the structure of the extended logistic loss function in the coming Section. 

\begin{figure*}[t]
\centering
\includegraphics[width=5.0in]{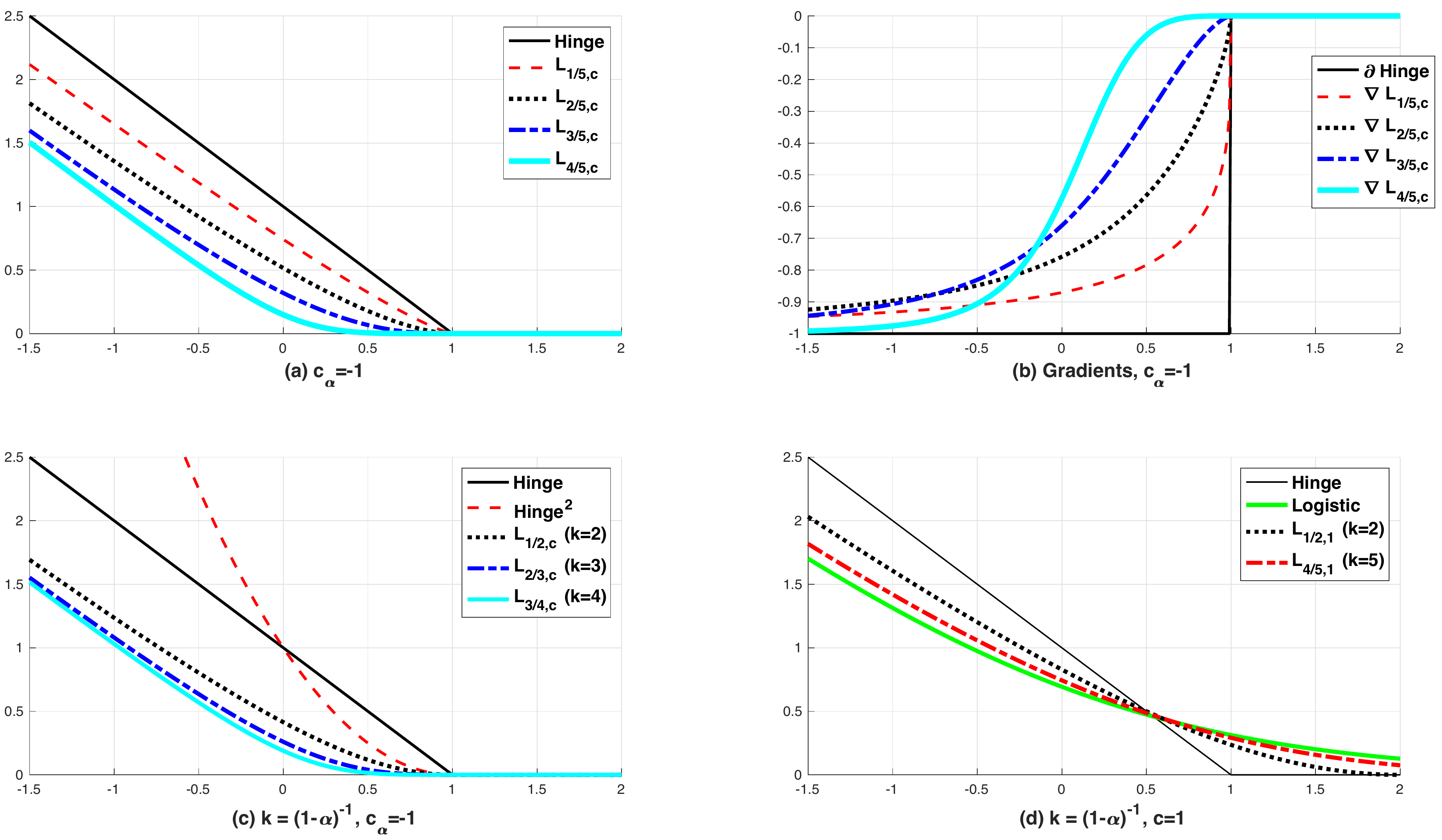}
\caption{Comparison of the Logitron loss with $\alpha<1$ and the well-known higher-order hinge loss in \eqref{hingeloss} and logistic loss in \eqref{clogistic}. (a) shows the relation between the Hinge-Logitron with $\alpha = i/5, \; i=1,2,3,4$ and $c_{\alpha}=-1$ and the hinge loss in \eqref{hingeloss}. (b) shows the gradient (or subgradient) of the loss functions in (a). (c) shows the relation between the Hinge-Logitron with $c_{\alpha}=-1$ and $\alpha=1/2,2/3,3/4$ (i.e., $k$-th order hinge loss with $k$-th root stabilizing function. Here $k=2,3,4$) and the first-order hinge loss and the second-order hinge loss. (d) shows the relation between Logistic-Logitron with $c=1$ and $\alpha=1/2, 4/5$ and the logistic loss and the hinge loss function. It is quite easy to identify the role of the margin parameter $c$ (or $c_{\alpha}$) and the model parameter $\alpha$ in the Logitron loss function.}
\label{fig:img3}
\end{figure*}

\begin{figure*}[t]
\centering
\includegraphics[width=5.0in]{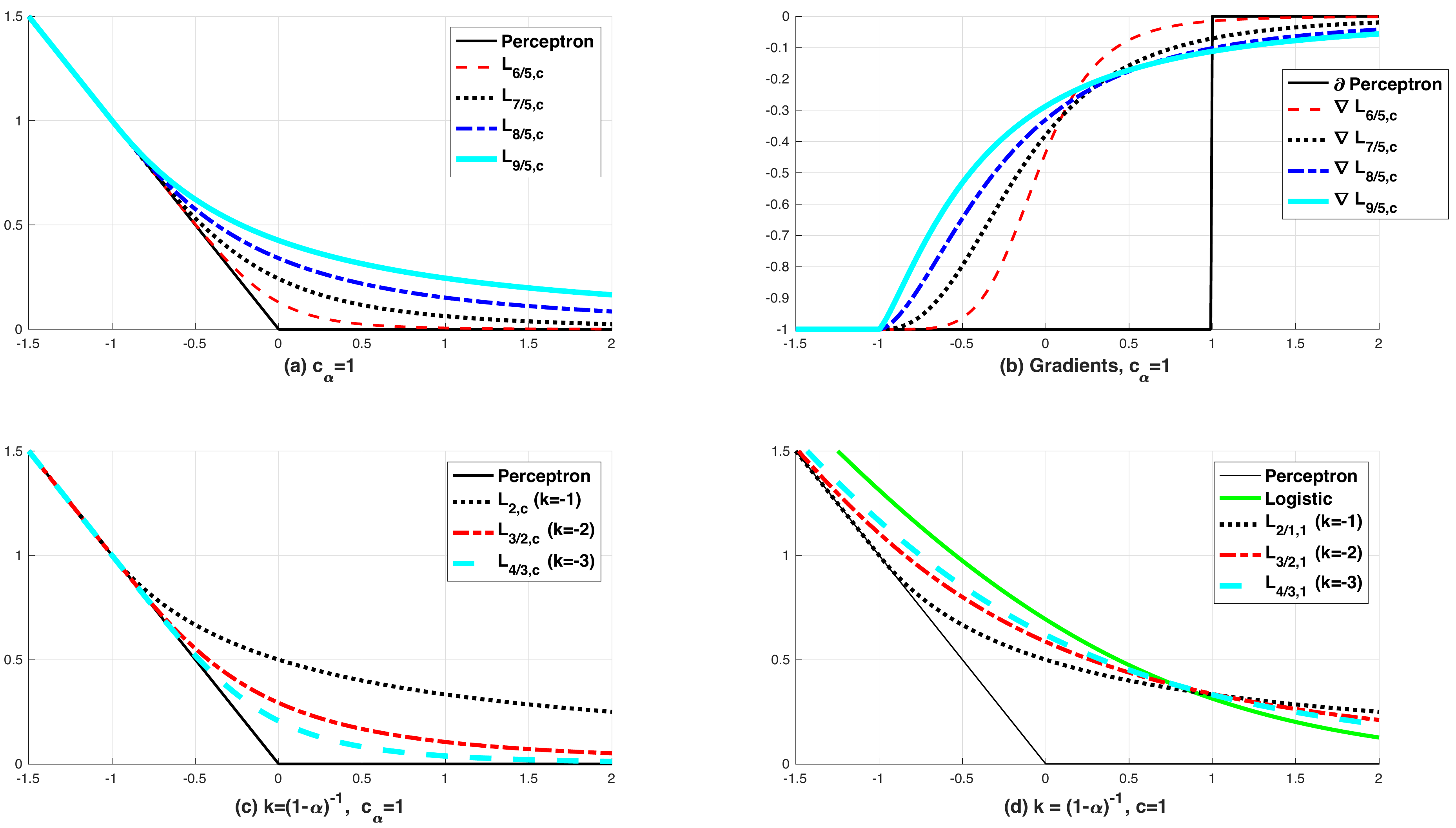}
\caption{Comparison of the Logitron loss with $\alpha>1$ and the well-known Perceptron loss in \eqref{maxloss} and logistic loss in \eqref{clogistic}. (a) shows the relation between the Hinge-Logitron with $\alpha = i/5, \; i=6,7,8,9$ and $c_{\alpha}=1$ and the Perceptron loss in \eqref{maxloss}. (b) shows the gradient (or subgradient) of the loss functions in (a). (c) shows the relation between the Hinge-Logitron with $c_{\alpha}=1$ and $\alpha=2,3/2,4/3$ (i.e., $k$-th order hinge loss with $k$-th root stabilizing function. Here $k=-1,-2,-3$) and the Perceptron loss function. (d) shows the relation between Logistic-Logitron with $c=1$ and $\alpha=2, 3/2,4/3$ and the logistic loss and the Perceptron loss function. When we set $c_{\alpha}=1$, the Logitron loss with $\alpha>1$ is a kind of negative higher-order hinge loss with stabilizer and thus the meaning of the margin is not that of the classic hinge loss. In fact, the margin is located on the negative Perceptron line.  On the other hand, when we set $c=1$, like the Logitron with $\alpha<1$, the Logitron with $\alpha>1$ also approximately converges to the logistic loss function as $\alpha \rightarrow 1$ (or $k \rightarrow -\infty$).}
\label{fig:img4}
\end{figure*}

\section{Logitron: An extended Logistic regression classification model augmented with the Perceptron\label{sec3}}
This Section introduces a general classification framework. That is, the Logitron classification model with the Perceptron-augmented extended logistic loss function. 


Let us start with the extended logistic loss function, which is a simple combination of $\exp_{\alpha,c}$ and $\ln_{\alpha,c}$ in the logistic regression style. In fact, it covers many loss functions appearing in classification such as exponential loss, (extended) sigmoid function, the Savage loss function and so on.
\begin{definition}\label{extlossDef}
Let $\alpha,\beta \in \R_+$ and $c \in \R_{++}$. Then the extended logistic loss function $\ell_{\alpha,\beta,c} : \dom(\ell_{\alpha,\beta,c}) \rightarrow \R$ is defined as
\begin{equation}\label{extloss}
\ell_{\alpha,\beta, c}(x) = \ln_{\alpha,c}(c + \exp_{\beta,c}(-x))
\end{equation}
where $\dom(\ell_{\alpha,\beta,c}) = \{ x \in \R \;|\;  -x \in \dom(\exp_{\beta,c})\; \}$.  Note that $\dom(\exp_{\beta,c})$ is the restricted domain in Table \ref{table7}. 
\end{definition}
By virtue of Lemma \ref{lemmax}, the extended logistic loss in \eqref{extloss} is well defined with the restricted domain in Table \ref{table7}, irrespective of choices of $\alpha,\beta \in \R_+$ and $c \in \R_{++}$. The classic logistic loss~\eqref{clogistic} is recovered when we set $\alpha=\beta=1$, irrespective of the choice of the auxiliary parameter $c$. 
Since we do not put any constraints on $\alpha$ and $\beta$, it is questionable when the extended logistic loss $\ell_{\alpha,\beta,c}(x)$~\eqref{extloss} is acting like the conventional logistic loss function~\eqref{clogistic}. The following theorem gives a partial answer in terms of convexity of $\ell_{\alpha,\beta,c}$~\eqref{extloss}.
\begin{theorem}
Let $\alpha,\beta \in \R_+$ with $\beta \ge \alpha$ and $c \in \R_{++}$. Then the extended logistic loss function $\ell_{\alpha,\beta,c}$~\eqref{extloss} is convex on $\dom(\ell_{\alpha,\beta,c}) = \{ x \in \R \;|\; -x \in \dom(\exp_{\beta,c}) \}$. 
\end{theorem}
\begin{proof}
Let us assume that $p \in int(\dom(\ell_{\alpha,\beta,c}))$. Then $h(p)=\exp_{\beta,c}(-p)>0$ and $\frac{h(p)}{c + h(p)} \in (0,1)$. For all $\beta \ge \alpha \ge 0$, we have
\begin{equation*}
\frac{d^2\ell_{\alpha,\beta,c}}{dp^2} = \left[\beta(c+h(p)) - \alpha h(p) \right]\frac{h(p)^{2\beta-1}}{(c+h(p))^{\alpha+1}} \ge 0
\end{equation*}
Now, we only need to extend convexity to the boundary $\dom(\ell_{\alpha,\beta,c})\cap bd(\dom(\ell_{\alpha,\beta,c})).$ From Table \ref{table7}, we have $-c_{\beta}\in \dom(\ell_{\alpha,\beta,c})\cap bd(\dom(\ell_{\alpha,\beta,c}))$ when $0 \le \beta < 1$. In fact, from the convexity of $\ell_{\alpha,\beta,c}$, we have
$$
\ell_{\alpha,\beta,c}(\lambda a+ (1-\lambda)b) \le \lambda \ell_{\alpha,\beta,c}(a) + (1-\lambda)\ell_{\alpha,\beta,c}(b)
$$
where $\lambda \in [0,1]$, $a<b$ and $a,b \in int(\dom(\ell_{\alpha,\beta,c})).$ By sending $b \rightarrow -c_{\beta}$, we can easily extend convexity up to the $\hbox{dom}(\ell_{\alpha,\beta,c})$.
\end{proof}
Although the nonconvex extended logistic loss ($\beta<\alpha$) is not main concern of this work, it is worth mentioning about the nonconvex loss function. As observed in \cite{freund09}, a nonconvex loss has some advantages in terms of robustness against the label noise. Actually, various nonconvex loss functions are proposed in boosting~\cite{mason00,freund09,nguyen13,masnadi08}, most of them are a subclass of the extended logistic loss. 
In the following example, we demonstrate higher-order sigmoid function which is a typical example of the nonconvex extended logistic loss function~\eqref{extloss}. They are known as the robust loss function in boosting~\cite{masnadi08}  or activation function~\cite{collobert04,rosenblatt57} in (multilayer Perceptron) neural network.
\begin{example}[{\bf higher-order sigmoid function}]
Let us consider the extended logistic loss function with $\alpha=m+1$ ($m \in \mathbb{N}$) and $\beta=1$ (higher-order sigmoid function):
\begin{equation}\label{hsigmoid}
\ell_{m+1,1,c}(z) = c_{m+1}\left(1 - \frac{1}{[1+\exp(-z)]^{m}}\right) = c_{m+1}(1-\sigma^m(z))
\end{equation}
where $\sigma(z) = \frac{1}{1 + \exp(-z)}$ is a sigmoid function and $c_{m+1} = \frac{c^{-m}}{m}$. Note that $\ell_{m+1,1,c}(z) \in (0,c_{m+1})$ for all $z \in \R$. In fact, the Savege loss function~\cite{masnadi08} is the second-order sigmoid function ($\alpha=3$) and the activation function in multilayer Perceptron neural network~\cite{collobert04,rosenblatt57} is the first-order sigmoid  function ($\alpha=2$).
\begin{itemize}
\item First-order sigmoid ($\alpha=2,\beta=1$): $
\ell_{2,1,c}(z) = \sigma(-z)$ where $c=1$.
\item Second-order sigmoid ($\alpha=3,\beta=1$):
$\ell_{3,1,c}(z) = 1-\sigma^2(z)$
where $c=0.5^{0.5}$, $z \in \R$ and $\ell_{3,1,c}(z) \in (0,1)$ for all $z \in \R$. In \cite{masnadi08}, authors have introduced $\sigma^2(-z) = 1 - \ell_{3,1,c}(-z)$ as the Savage loss function in boosting. This model is known to be more robust to label noise compared to other boosting models having convex loss functions such as Adaboost~\cite{freund97} and LogitBoost~\cite{friedman00}. However, within the convex loss function, the LogitBoost with logistic loss is more robust than the Adaboost with the exponential loss~\cite{friedman00}. 
\end{itemize}
\end{example}
Since we are mainly interested in convex loss function, having similar features of the loss functions used in logistic regression and SVM, we restrict the extended logistic loss function~\eqref{extloss} by the following condition.
\begin{equation}\label{assumptionA}
\alpha = \beta \ge 0
\end{equation}
Now, let us simplify the notation of the extended logistic regression function with the extended-valued function by
\begin{equation}\label{extloss2}
\boxed{
\ell_{\alpha,c}(x) = \left\{ 
\begin{array}{l}
\ln_{\alpha,c}(c + \exp_{\alpha,c}(-x)) \quad \hbox{ if } x \in \dom(\ell_{\alpha,c})\\
+\infty \hskip 3.0cm \hbox{ otherwise }
\end{array}
\right.
}
\end{equation}
Here $\alpha \in \R_{+}$ is a model parameter and $c \in \R_{++}$ is a margin parameter. As observed in Figure \ref{fig:img3} and \ref{fig:img4}, the search space of two parameters are significantly reduced and thus they are not a big burden while running the cross-validation. The only concern of \eqref{extloss2} is that the domain $\dom(\ell_{\alpha,c}) = \{ x \in \R \;|\; -x \in \dom(\exp_{\alpha,c})\}$ depends on $c_{\alpha}$ (see Table \ref{table7}). This is definitely a barrier  for various applications appearing in machine learning. However, interestingly, the domain dependency problems of \eqref{extloss2} could be easily escaped by using the Perceptron loss function. We call the Perceptron-augmented loss function of \eqref{extloss2} as the Logitron loss function and the corresponding minimization model for classification as Logitron. The details are following.
\begin{definition}[{\bf Logitron}]\label{Def:Logitron}
Let $(x_1,y_1), ... , (x_N,y_N) \in {\cal X} \times {\cal Y}$ be the given training dataset. Here ${\cal X} = \{ x \in \R^n \;|\; \norm{x}_{\infty} \le R_{\cal X} \}$, ${\cal Y} = \{-1,+1\}$, and $R_{\cal X} \in \R_{++}$. Also, we set $(\alpha,c) \in \R_+\times\R_{++}$. Then we have the Logitron model:
\begin{equation}\label{LogiT}
\min_{f \in {\cal F}}\;\sum_{i=1}^N L_{\alpha,c}(y_if(x_i))  
\end{equation}
where ${\cal F}$ is an appropriate function space such as~\eqref{funcSpace} and $L_{\alpha,c} : \R \rightarrow \R_{+}$ is the Logitron loss function defined by
\begin{equation}\label{LogiTloss}
L_{\alpha,c}(z) = 
\left\{ 
\begin{array}{l}
\ell_{\alpha,c}(z) \qquad \hbox{ if } z \in \dom(\ell_{\alpha,c})\\
\ell_P(z) \qquad\; \hbox{ otherwise }
\end{array}
\right.
\end{equation}
where $\ell_P(z) = \max(0,-z)$ is the Perceptron loss function in~\eqref{maxloss} and $\dom(\ell_{\alpha,c}) = \{ z \in \R \;|\; -z \in \dom(\exp_{\alpha,c})\}.$
\end{definition}
Since the Perceptron loss is added to the extended logistic loss function, we have $\emptyset = \{ z \in \R  \;|\; \ell_{\alpha,c}(z) = +\infty \;\}$.  That is, the domain of the Logitron is the entire real line. 
Moreover, the Logitron loss is continuously twice differentiable on its entire domain $\R$ under the mild condition. See also Figure \ref{fig:img3} and \ref{fig:img4} for the graph of the Logistic loss and the gradient of it.
\begin{theorem}\label{chLogitron}
Let $c \in \R_{++}$. Then the Logitron loss function $L_{\alpha,c} : \R \rightarrow \R_{+}$~\eqref{LogiTloss} is convex and continuous for all $\alpha \in \R_{+}$. When $\alpha \in \R_{++}$, it is continuously differentiable. Moreover, if $\alpha \in (0.5,2)$ then it is continuously twice differentiable.  
\end{theorem} 
\begin{proof}
When $\alpha=1$, we get $\ell_{\alpha,c} = \ell_L$ the logistic loss in~\eqref{clogistic}. Thus, $\ell_{L}(z) \cap \ell_P(z) = \emptyset$ for all $z \in \R$ and $\ell_L(z)$ is infinitely differentiable on $\R$. Let us consider $\alpha \in \R_{+} \setminus \{ +1 \}$. 

Firstly, for the continuity of $L_{\alpha,c}$, we only need to show that
$$
\ell_{\alpha,c}(z) \cap \ell_P(z) =   
\left\{\begin{array}{l} 
 \emptyset   \hskip 1.5cm \hbox{ if } z \in int(\Omega) \\
-z \;\hbox{ or } \; 0  \quad \hbox{ if } z \in bd(\Omega) 
\end{array}
\right.
$$
where $-z \in \R_{++}$ and $\Omega = dom(\ell_{\alpha,c})$. 
\begin{itemize}
\item $0 \le \alpha<1$: We have $-c_{\alpha}>0$ and $dom(\ell_{\alpha,c}) = \R_{\le -c_{\alpha}}$. Thus, $bd(\Omega) = \{ -c_{\alpha} \}$. Therefore, we get
$
\ell_{\alpha,c}(-c_{\alpha}) = \ln_{\alpha,c}(c + \exp_{\alpha,c}(c_{\alpha})) = 0 = \ell_P(-c_{\alpha}) .
$
Additionally, since $\exp_{\alpha,c}$ is strictly increasing and convex (Proposition \ref{propexp} (1) and (2)), we have $\varepsilon = \exp_{\alpha,c}(-z) > 0$ for all $z \in int(\Omega) \cap \R_+$. Therefore, since $\ln_{\alpha,c}$ is strictly increasing and $\ln_{\alpha,c}(c)=0$,  we have $\ell_{\alpha,c}(z) = \ln_{\alpha,c}(c + \varepsilon) >0$. Additionally, for all $z \in int(\Omega) \cap \R_-$, we have $-z < \ell_{\alpha,c}(c + \exp_{\alpha,c}(-z))$ from $\exp_{\alpha,c}(-z) < c + \exp_{\alpha,c}(-z)$.
\item $\alpha>1$: We have $c_{\alpha}>0$ and $dom(\ell_{\alpha,c}) = \R_{> -c_{\alpha}}$. Thus  $bd(\Omega) = \{ -c_{\alpha} \}$. Since $bd(\Omega) \cap \Omega = \emptyset$, we need to be cautious on the boundary point. From the extended-valued real number system, we have $(c + \infty)_{\alpha} = 0$ and thus 
\begin{eqnarray*}
\ell_{\alpha,c}(-c_{\alpha}) = \ln_{\alpha,c}(c + \exp_{\alpha,c}(c_{\alpha}))
= \ln_{\alpha,c}(c + \infty)
= c_{\alpha}
\end{eqnarray*}
Note that it is easy to check that $\exp_{\alpha,c}(-z) > 0$ for all $z \in int(\Omega) \cap \R_+$ and 
$-z < \ell_{\alpha,c}(c + \exp_{\alpha,c}(-z))$ for all $z \in int(\Omega) \cap \R_-$.
\end{itemize}

Secondly, we will show continuously differentiability of the Logitron, $L_{\alpha,c}(z)$ on its entire domain $\R$. 
\begin{itemize}
\item $0 < \alpha<1$: $dom(\ell_{\alpha,c}) = \R_{ \le - c_{\alpha}}$ and $bd(\dom(\ell_{\alpha,c})) = \{ -c_{\alpha} \}$. By simple calculation, we have
\begin{equation}\label{temp1}
L'_{\alpha,c}(z) = 
\left\{\begin{array}{l} \ell'_{\alpha,c}(z)
\quad\hbox{ if } z \in int(\dom(\ell_{\alpha,c}))\\
0  \hskip 1.4cm \hbox{ if } z \in \R \setminus \dom(\ell_{\alpha,c})
\end{array}\right.
\end{equation}
where $\ell'_{\alpha,c}(z) = -\left(\frac{\exp_{\alpha,c}(-z)}{c + \exp_{\alpha,c}(-z)}\right)^{\alpha}<0,\; \forall z \in int(\dom(\ell_{\alpha,c})).$  Since $\alpha \in (0,1)$, as $z \rightarrow -c_{\alpha} \in bd(\dom(\ell_{\alpha,c}))$, we get $\ell'_{\alpha,c}(z) \rightarrow 0$ and $\ell'_P(-c_{\alpha}) = 0$. Therefore, $L'_{\alpha,c}(z)$ is well defined for all $z \in \R$.
\item $\alpha>1$: $dom(\ell_{\alpha,c}) = \R_{> -c_{\alpha}}$ and $c_{\alpha}>0$. Since $int(\dom(\ell_{\alpha,c})) = \dom(\ell_{\alpha,c})$, we have
\begin{equation}\label{temp2}
L'_{\alpha,c}(z) = 
\left\{\begin{array}{l} 
\ell'_{\alpha,c}(z)
\quad\hbox{ if } z \in \dom(\ell_{\alpha,c})\\
-1  \hskip 1.2cm \hbox{ otherwise}
\end{array}\right.
\end{equation}
where $\ell'_{\alpha,c}(z) = -\left(\frac{\exp_{\alpha,c}(-z)}{c + \exp_{\alpha,c}(-z)}\right)^{\alpha}>-1$ for all $z \in \dom(\ell_{\alpha,c})$. On the other hand, since $\exp_{\alpha,c}(c_{\alpha}) = \infty$ at $-c_{\alpha} \in bd(\dom(\ell_{\alpha,c}))$, we have $\ell'_{\alpha,c}(-c_{\alpha}) = \left(\frac{\infty}{c+\infty}\right)^{\alpha} = -1$. 
\end{itemize}

Thirdly, for continuously twice differentiability of the Logitron loss, 
let us take the second derivative of $\ell_{\alpha,c}$. Then, $\forall\alpha \in (0.5,2)$ and $z \in int(\dom(\ell_{\alpha,c}))$ 
$$
\ell''_{\alpha,c}(z) = c\alpha \exp_{\alpha,c}(-z)^{\alpha-2}  \left(\frac{\exp_{\alpha,c}(-z)}{c + \exp_{\alpha,c}(-z)} \right)^{\alpha+1} > 0.
$$
Let us consider the case $\alpha \in (0.5,1)$ and $(1,2)$.
\begin{itemize}
\item $0.5<\alpha<1$: From $bd(\dom(\ell_{\alpha,c}))  \cap \dom(\ell_{\alpha,c}) = \{ -c_{\alpha}\}$, we get
$$
\ell''_{\alpha,c}(-c_{\alpha}) =\frac{ c\alpha \exp^{2\alpha-1}_{\alpha,c}(c_{\alpha})}{(c+\exp_{\alpha,c}(c_{\alpha}))^{\alpha+1}} = 0
$$
\item $1<\alpha<2$: From $bd(\dom(\ell_{\alpha,c})) = \{ -c_{\alpha} \}$ and $\exp_{\alpha,c}(z) = c (1- \frac{z}{c_{\alpha}})^{1/(1-\alpha)}$, we have
$
\exp_{\alpha,c}(-z)^{\alpha-2} = c^{\alpha-2}\left(1 - \frac{-z}{c_{\alpha}}\right)^{\frac{\alpha-2}{1-\alpha}}.
$
Thus, $\ell''_{\alpha,c}(-c_{\alpha}) =  0$.
\end{itemize}
Additionally, it is trivial that $\ell_P''(z) = 0,\; \forall  z \in \R \setminus \{ 0 \}$. Finally, $\forall\alpha \in (0.5,2)$ and $\forall z \in \R$, we get the continuous second derivative of the Logitron loss function
$$
L''_{\alpha,c}(z) = 
\left\{\begin{array}{l}
\ell''_{\alpha,c}(z) \quad \hbox{ if } z \in \dom(\ell_{\alpha,c})\\
0 \hskip 1.4cm \hbox{ otherwise}
\end{array}\right.
$$
\end{proof}
Due to the Theorem \ref{chLogitron}, the Logitron loss function can be used as a classification loss function for all $\alpha \in \R_+$ and $c \in \R_{++}$. In fact, it is classification-calibrated~\cite{bartlett06}. 
\begin{corollary}\label{calibration}
For all $\alpha \in \R_+$ and $c \in \R_{++}$, the Logitron loss function $L_{\alpha,c}(z)$ is classification-calibrated~\cite{bartlett06}.
\end{corollary}
\begin{proof}
From Theorem \ref{chLogitron}, $L_{\alpha,c}(z)$ is convex and differentiable at $z=0$ for all $\alpha \in \R_+$.
$$
L'_{\alpha,c}(0)= 
\left\{\begin{array}{l}  
-\left(\frac{\exp_{\alpha,c}(0)}{c + \exp_{\alpha,c}(0)}\right)^{\alpha} = -\frac{1}{2^{\alpha}} , \qquad \hbox{ if } \alpha>0\\
-1 \hskip 4.1cm \hbox{ if } \alpha=0
\end{array}\right.
$$
Therefore, we have $L'_{\alpha,c}(0) = - \frac{1}{2^{\alpha}} < 0$ for all $\alpha \ge 0$. Hence $L_{\alpha,c}(z)$ is classification-calibrated, irrespective of the choice of $c \in \R_{++}$. 
\end{proof}
Additionally, the Logitron loss function~\eqref{LogiTloss} is sufficiently smooth. That is, the gradient of it is continuous on its entire domain $\R$ and bounded by one. Therefore, we could use any gradient-based optimization method such as L-BFGS~\cite{mark19}. 
\begin{corollary}\label{Lipschitz}
For all $\alpha \in \R_+$ and $c \in \R_{++}$, the Logitron loss function $L_{\alpha,c}(z)$ is Lipschitz continuous with Lipschitz constant one for all $z \in \R$. That is, we have
\begin{equation}
\abs{L_{\alpha,c}(z_1) - L_{\alpha,c}(z_2)} \le \abs{z_1 - z_2}
\end{equation}
for all $z_1, z_2 \in \R$.
\end{corollary}
\begin{proof}
Let $\alpha \in \R_{++}$. From \eqref{temp1} and \eqref{temp2}, we have $\abs{L'_{\alpha,c}(z)} \le 1$. Moreover, when $\alpha=0$, we get $\abs{\partial L_{\alpha,c}(z)} \le 1$. Here $\partial L_{\alpha,c}(z)$ is a subgradient of $L_{\alpha,c}$.
\end{proof}

Before we go further, it is worth mentioning about the (un)hinge loss function. The extended logistic loss function with $\alpha=0$ has an unconventional hinge loss function, known as unhinged loss function~\cite{rooyen15}. In fact, the extended logistic loss under the domains in Definition \ref{elogf} and \ref{eexpf} becomes
$\ell_{0,c}(z) = [c+(-z + c)] - c = c - z$
where $-z \in \R$. The main advantage of this unhinged loss function is that it is robust to symmetric label noise. In fact, as observed in  \cite{rooyen15}, if the convex function is lower bounded, then it is not robust  to symmetric label noise. However, the Logitron with $\alpha=0$ is the hinge loss function which can be reformulated with first-order hinge loss function in \eqref{hingeloss}:
\begin{equation}\label{hingeX}
L_{0,1}(z) = \max(0,1 - z) 
\end{equation}
The Logitron with $0<\alpha<1$ can be regarded as the smoothed hinge loss when we set $c_{\alpha}=-1$. Actually, the $k$-th order hinge loss in~\eqref{hingeloss} with an additional $k$-th root stabilizer function, is a special case of the Logitron with $0<\alpha<1$ and $c_{\alpha}=-1$. Also, as observed in Figure \ref{fig:img3} and \ref{fig:img4}, the Logitron loss function with $c=1$ behaves like the logistic loss function when $\alpha \approx 1$. 
Therefore, it is natural to separate the Logitron into the two category; one is the hinge-like Logitron loss function and the other is the logistic-like Logitron loss function. In the coming Section \ref{sec4}, we analyze the Logitron model in two different points of view. 

\section{The Low complexity Logitron with $\alpha \in (\R_{+} \setminus \{ 1 \}) \cap \mathbb{Q}$\label{sec4}}
In the previous Section, we found that the Logitron has many useful properties such as smoothness and classification-calibration. However, to be more practical in terms of computation, we need to reduce the spaces of model parameter $\alpha$ and of margin parameter $c$. In this Section, we introduce the low complexity Logitron loss function~\eqref{LogiTloss} with $\alpha \in \R_+ \cap \mathbb{Q}$ based on higher-order hinge loss in~\eqref{hingeloss}. With additional restriction of $\alpha$ and $c$, we have two different categories of the Logitron; one is the hinge loss-like Logitron (Hinge-Logitron) and the other is the logistic loss-like Logitron (Logistic-Logitron). 

Let us start with the generalized $q$-th order hinge loss function. As stated in \cite{friedman91}, it corresponds to a basis function of the generalized $q$-th order spline.
\begin{equation}\label{qhinge}
\ell_{H,q}(z) = (\max(0,1-z))^q
\end{equation}
where $q \in \mathbb{Q}^* = \{ \frac{a}{b} \;|\; \frac{a}{b} \not\in [0,1) \hbox{ and } a,b \in \mathbb{Z} \}$. Interestingly, the low complexity Logitron with $\alpha \in (\R_+ \setminus \{ 1\}) \cap \mathbb{Q}$ can be easily reformulated with the generalized $q$-th order hinge loss in \eqref{qhinge}. In fact, let us modify the extended exponential function $\exp_{\alpha,c}(z)$ with $\max(0,\cdot)$:
\begin{equation}\label{Eexp}
\hbox{Exp}_{\alpha,c}(z) = (\max(0,c^{1-\alpha} + (1-\alpha)z))^{1/(1-\alpha)}
\end{equation} 
where $\alpha \in (\R_+ \setminus \{ 1 \}) \cap \mathbb{Q}$.  Then we have a connection between \eqref{qhinge} and \eqref{Eexp}:
\begin{equation}\label{exphinge}
\hbox{Exp}_{\alpha,c}(z) =  c \ell_{H,\frac{1}{1-\alpha}}\left(\frac{z}{c_{\alpha}} \right)
\end{equation}
If we set $\abs{c_{\alpha}}=1$ and $q=\frac{1}{1-\alpha}$ then $\hbox{Exp}_{\alpha,c}$ becomes the generalized $q$-th order hinge loss function with an additional margin control parameter $c$. Now, let us reformulate the Logitron with the modified extended exponential function $\hbox{Exp}_{\alpha,c}$~\eqref{Eexp}. 
\begin{theorem}\label{qLogitron}
Let $\alpha \in (\R_+ \setminus \{ 1 \}) \cap \mathbb{Q}$ and $c \in \R_{++}$.
Then the low complexity Logitron can be reformulated with $\hbox{Exp}_{\alpha,c}(\cdot)$~\eqref{Eexp}:
\begin{equation}\label{lcLogitron}
L_{\alpha,c}(z) = 
\left\{\begin{array}{l} 
{\cal L}_{\alpha,c}(z) \hskip 2cm\hbox{ if } \alpha \in [0,1)\\
\max\left(-z, {\cal L}_{\alpha,c}(z) \right) \quad \hbox{ if } \alpha \in (1,\infty)
\end{array}\right.
\end{equation}
where 
\begin{equation}\label{mLogitron}
{\cal L}_{\alpha,c}(z) = \ln_{\alpha,c}(c + \hbox{Exp}_{\alpha,c}(-z))
\end{equation}
\end{theorem}
\begin{proof}
When $\alpha=0$, we get the first-order hinge loss in~\eqref{hingeloss}. Now, let us assume that $\alpha \in (0,1)$, then we have $-z \in dom(\exp_{\alpha,c}) = \R_{\ge c_{\alpha}}$. In this region, it is easy to see $\exp_{\alpha,c}(-z) = \hbox{Exp}_{\alpha,c}(-z)$. Also, when $-z \not \in \dom(\exp_{\alpha,c})$, we get $\hbox{Exp}_{\alpha,c}(-z) = 0$ and, from the definition of the Logitron loss in \eqref{LogiTloss}, $\ell_P(z) = 0$. Now let use assume that $\alpha > 1$. For all $-z \in \dom(\exp_{\alpha,c}) = \R_{< c_{\alpha}}$, it is easy to check $\hbox{Exp}_{\alpha,c}(-z)  = \exp_{\alpha,c}(-z)$. On the other hand, let $-z \not\in \dom(\exp_{\alpha,c}) = \R_{< c_{\alpha}}$, then $-z \ge c_{\alpha}$ and thus we have
$\hbox{Exp}_{\alpha,c}(-z) = c\ell_{H,\frac{1}{1-\alpha}}\left(\frac{-z}{c_{\alpha}} \right) = \left(\frac{1}{c\max\left(0,1-\frac{(-z)}{c_{\alpha}}\right)}\right)^{\frac{1}{\alpha-1}}  = \left(\frac{1}{0}\right)^{\frac{1}{\alpha-1}} = \infty$
and $\ln_{\alpha,c}(c + \hbox{Exp}_{\alpha,c}(-z)) =  \ln_{\alpha,c}(c + \infty) = \frac{1}{1-\alpha} \left(\frac{1}{c + \infty}\right)^{\alpha-1} + c_{\alpha} = c_{\alpha}$. Here, with an additional $\max(-z,\cdot)$ function, we have
$
\max(-z,{\cal L}_{\alpha,c}(z)) =  \max(-z,c_{\alpha}) = -z.
$ 
Note that if $z<0$ then $\ell_P(z) = \max(0,-z) = -z$. Therefore, when $z \not\in \dom(\ell_{\alpha,c}) = \R_{> -c_{\alpha}}$ (i.e., $z \le -c_{\alpha}$), we have  
$
\max(-z,{\cal L}_{\alpha,c}(z)) = L_{\alpha,c}(z).
$
\end{proof} 
Note that, in Theorem \ref{qLogitron}, though we have restricted the range of $\alpha$ for practical reason, it can be extended to $\R_+ \setminus \{ 1 \}$. When we set $\abs{c_{\alpha}}=1$, the Logitron loss is similar to the generalized $q$-th order hinge loss function~\eqref{qhinge}. However, if we set $c=1$ then the role of the extended exp and log is the approximation of the conventional exp and log function. Especially, when $\alpha \approx 1$, the Logitron loss function almost equals to the logistic loss function. See Figure \ref{fig:img3} and \ref{fig:img4}. As a consequence, we have four different categories of the Logitron loss function based on the model parameter $\alpha$ and the margin parameter $c$.
\begin{itemize}
\item { Hinge-Logitron} ($\abs{c_{\alpha}}=1$): { H-Logitron} ($0<\alpha<1$ and $c_{\alpha}=-1$) and { H+Logitron} ($\alpha>1$ and $c_{\alpha}=1$)
\item { Logistic-Logitron} ($c=1$):  { L-Logitron} ($0<\alpha<1$ and $c=1$) and  { L+Logitron} ($\alpha>1$ and $c=1$)
\end{itemize}
Since the parameter $q$ of the generalized $q$-th order hinge-loss function~\eqref{qhinge} can be negative, the classic margin concept is also required to be generalized. We call the classic margin as {\it positive margin} if the loss function touch the Perceptron loss on the positive axis. On the other hand, if the loss function touch the Perceptron loss on the negative axis, then we call that touch point as the {\it negative margin}. Actually, the positive margin ($\alpha<1$) and negative margin ($\alpha>1$) equals to the value of $\abs{c_{\alpha}}$ (i.e., $bd(\dom\ell_{\alpha,c})$). Therefore, since the logistic regression does not touch the Perceptron loss, it does not have margin. In Hinge-Logitron, the H-Logitron loss function has positive margin like the higher-order hinge loss~\eqref{hingeloss}. Figure \ref{fig:img3} (c) compare the H-Logitron with the first-order hinge-loss (SVM) and the second order hinge-loss (L2SVM). However, the H+Logitron loss function has negative margin through the Perceptron line (i.e., $\ell_P(z)$). See Figure \ref{fig:img4} (c) for the shape of the H+Logitron with various different choice of $\alpha = i/5$ and $i=6,7,8,9$. As regards the Logistic-Logitron model, we have the L-Logitron loss function approximating the logistic loss with positive margin and the L+Logitron loss function approximating the logistic loss with negative margin. See Figure \ref{fig:img3} (d) and Figure \ref{fig:img4} (d), respectively.

It is useful seeing a  direct connection between higher-order hinge loss in~\eqref{hingeloss} and the Logitron loss function $L_{\alpha,c}(z)$ with $\frac{1}{1-\alpha} = k \in  \mathbb{Z} \setminus \{ 0 \}$. Here $\alpha = 1-k^{-1} \in [0,2]$. Then \eqref{exphinge} is simplified as $\hbox{Exp}_{1-k^{-1},c}(z) = c\ell_{H,k}\left(\frac{z}{c_{k}}\right)$ 
with $\alpha = 1- k^{-1}$ and $c_k = -kc^{1/k}<0$. Now, when $k \ge 1$ (i.e., $\alpha \in [0,1)$), we have H-Logitron ($c_k=-1$)
\begin{equation}\label{hlogitronF}
f^* = \argmin_{f \in {\cal F}} L_{\alpha,c}(yf(x)) = \argmin_{f \in {\cal F}} \sqrt[k]{1 + \ell_{H,k}(yf(x))}
\end{equation}
It actually means that the H-Logitron with $\alpha=1-k^{-1}$ and $k \in \mathbb{Z} \setminus \{ 0 \}$ is a higher-order SVM with an additional $k$-th root stabilizer function. As observed in Figure \ref{fig:img3} (c), the second-order hinge-loss (L2SVM) highly penalize the misclassified data. On the other hand, the penalty on the misclassified data of the H-Logitron is  stabilized, irrespective of the choice of $k>0$. 

When $k\le -1$ (i.e., $\alpha>1$), we get a totally new classification model, H+Logitron.
\begin{equation} 
\min_{f \in {\cal F}} L_{\alpha,c}(yf(x)) 
\end{equation}
where $c_{k} = -k c^{1/k}>0$ and
$$
L_{\alpha,c}(z)  = \max\left(-z, c_{k}-c_k\left\{ 1 + \ell_{H,k}\left(-\frac{z}{c_{k}}\right) \right\}^{1/k} \right)
$$
  In this instance,  we do not have positive margin. That is, $L_{\alpha,c}(z) \cap \ell_P(z) = -z$ for all $z \le -c_k$ and $L_{\alpha,c}(z)>0$ for all $z > -c_k.$ By controlling $c_{k}$ (i.e., the negative margin), we obtain the closeness of the H+Logitron $L_{\alpha,c}$ to the Perceptron loss function $\ell_P$. Though the H+Logitron does not have the classic margin, i.e., the positive margin, however, due to its simple structure of the model, we need to investigate the H+Logitron model in more details.  For instance, let $k=-1 (\alpha=2)$ then we have
\begin{equation}\label{lowcomplexityLoss}
L_{2,c}(z) = \left\{\begin{array}{l} 
\frac{c_{-1}}{2 + z/c_{-1}} \quad \hbox{ if } z > -c_{-1}\\
-z \hskip 1.1cm \hbox{ otherwise }
\end{array}\right.
\end{equation}
where $c_{-1}= c^{-1}$. Interestingly, we can remove singularity which existed on the boundary of the domain of the extended exponential function. Moreover, as noticed in Theorem \ref{chLogitron}, H+Logitron with $k=-1$ has a continuous derivative, $
\frac{dL_{2,c}(z)}{dz} = \max\left(-1, -(2+cz)^{-2}\right)$. 
The most important feature of \eqref{lowcomplexityLoss} is that we only need division and multiplication for the evaluation of the gradient and the loss function itself. This is the main advantage of \eqref{lowcomplexityLoss}. As observed in Section \ref{sec5}, the performance of it is comparable to logistic regression and SVM. Note that, when $k=-2 (\alpha=3/2)$, the H+Logitron can be reformulated as
\begin{equation}\label{lowcomplexityLoss2}
L_{3/2,c}(z) = \max\left(-z, c_{-2}\left\{ 1 - \frac{\abs{1 + z/c_{-2}}}{\sqrt{1 + (1+ z/c_{-2})^2}} \right\} \right)
\end{equation}
This model is rather complicated. However, it is also smooth on the entire domain $\R$ and classification-calibrated. In fact, when $z \gg c_{-2}$, $L_{3/2,c}(z) \approx 0$ and, when $z\le-c_{-2}$, $L_{3/2,c}(z) = -z$. It behaves like the conventional margin-based loss function. 

\section{Experiments with various $\ell_2$-regularized Logitron models\label{sec5}}
This Section compare performance of the proposed Logitron with logistic regression and SVM within the linear classification framework. 

Let us define the Logitron minimization problem with the linear function space in \eqref{funcSpace}. For simplicity, we use $\ell_2$-regularizer, but it could be replaced with a sophisticated regularization model.
\begin{equation}\label{linmin}
\min_{w,b}\; H(w,b)+ \lambda Reg(w)
\end{equation}
where $Reg(w) = \norm{w}_2^2 = \inprod{w}{w}$ is the $\ell_2$-regularizer,  
$
H(w,b) = \sum_{i=1}^N L_{\alpha,c}(y_i[\inprod{w}{x_i} + b]), 
$
and
$(x_i , y_i) \in \R^n \times \{-1,+1 \}$. 
 Although the loss function $H$ is rather complicated, it has many useful properties for gradient-based optimization. Indeed, the loss function $H$ is convex and differentiable on $\R^{n+1}$, irrespective of the choice of $\alpha \in \R_{++}$. For simplicity, we use the L-BFGS algorithm in {\it minFunc}~\cite{mark19}. It is implemented in the MATLAB framework. Note that we use the famous LIBLINEAR package~\cite{fan08} for the benchmark of the proposed Logitron model. Among various linear classification models in  LIBLINEAR, we select typical models; logistic regression~\eqref{clogistic} and higher-order SVM~\eqref{hingeloss} (the first-order SVM and the second-order SVM (i.e., L2SVM)). For logistic regression, we use the primal formulation ($s=0$). For SVM, we use the dual formulation ($s=3$). For L2SVM, we use the primal formulation ($s=2$). We also use the bias term in LINLINEAR ($B=1$). Note that all models have $\ell_2$-regularization term. 
As regards the regularization parameter $\lambda$, we simply use the following parameter selection strategy for $\lambda$ as recommended in the LIBSVM~\cite{chang11}.
\begin{equation}\label{lambdaD}
\lambda = 2^{d}, \; d = -14,-13,-12,...,5
\end{equation} 
In the models of LIBLINEAR, the regularization parameter is located on the loss function and thus we use $\lambda^{-1}$ of \eqref{lambdaD} for the regularization parameter of them. 

\begin{table*}
\centerline{
\begin{tabular}{c|c||l|l||l|l||l}
\hline
\textbf{Model} &\textbf{Submodel}&${\bm \alpha}$& ${\bm q} = \frac{1}{1-\alpha}$&${\bm c_{\bm \alpha}}$&${\bm c}$&${\bm \lambda}$\\\hline\hline
\textbf{H-Logitron} &\textbf{H-1} &$1/5,\; 2/5,\; 3/5,\; 4/5$& $5/4,5/3,5/2,5$ &$-1$& - & \eqref{lambdaD} \\\cline{2-7}
&\textbf{H-2} &$1/2$&$2$&$-1,-4/4,-3/5,-2/5$& - & \eqref{lambdaD} \\\cline{2-7} 
&\textbf{H-3} &$2/3$&$3$&$-1,-4/4,-3/5,-2/5$& - & \eqref{lambdaD} \\\cline{2-7} 
&\textbf{H-4}&$3/4$&$4$&$-1,-4/4,-3/5,-2/5$& - & \eqref{lambdaD} \\\hline
\textbf{H+Logitron}  & {\bf H+1} &$6/5,\;7/5,\;8/5,\;9/5$&$-5,-5/2,-5/3,-5/4$&$1$& - & \eqref{lambdaD} \\\cline{2-7}
& {\bf H+2} &$2$&$-1$&$1,4/5,3/5,2/5$& - & \eqref{lambdaD} \\\cline{2-7}
&\textbf{H+3} &$3/2$&$-2$&$1,4/5,3/5,2/5$& - & \eqref{lambdaD}  \\\hline 
{\bf L-Logitron} &\textbf{L-} &$4/5,\;5/6,\;7/8,\;11/12$&$5,6,8,12$&-& 1 & \eqref{lambdaD} \\\hline
\textbf{L+Logitron}&\textbf{L+}&$4/3,\;5/4,\;8/7,\;13/12$&$-3,-4,-7,-12$& - & $1$ &\eqref{lambdaD}  \\\hline 
\end{tabular}
}
\caption{The parameter spaces of nine sub-Logitron models; four H-Logitron (H-1,H-2,H-3, and H-4), three H+Logitron (H+1, H+2, and H+3), L-Logitron (L-), and L+Logitron (L+).  
}\label{tablePar}
\end{table*}

In terms of parameter space of Logitron, we need to select not only the regularization parameter $\lambda$ but also the model parameter $\alpha$ and the margin parameter $c$. From the analysis in the earlier Section \ref{sec4}, we know that the Logitron has four different submodels (H-Logitron, H+Logitron, L-Logitron, and L+Logitron). The H-Logitron is the higher-order SVM with an additional stabilization function~\eqref{hlogitronF}. For simplicity, we only consider $2$th - $5$th order SVM with the corresponding $k$-th root function. 
In the category of H+Logitron, we have two sub-models; H+Logitron with $\alpha=2 (k=-1)$~\eqref{lowcomplexityLoss} and H+Logitron with $\alpha=3/2 (k=-2)$~\eqref{lowcomplexityLoss2}. Actually, the minimization problem of the H+Logitron with $\alpha=2$ can be solved by using elementary arithmetics such as division and multiplication. In total, we have nine sub-models; {\bf H-1}($\alpha=i/5$ with $i=1,2,3,4$ i.e., $k=2,3,4,5$), {\bf H-2} ($\alpha=1/2$, i.e., $k=2$), {\bf H-3} ($\alpha=2/3$, i.e., $k=3$), {\bf H-4} ($\alpha=3/4$, i.e., $k=4$), {\bf H+1} ($\alpha=j/5$ with $j=6,7,8,9$), {\bf H+2} ($\alpha=2$), {\bf H+3} ($\alpha=3/2$), {\bf L-} ($\alpha=4/5,5/6,7/8,11/12$), and {\bf L+} ($\alpha=4/3,5/4,8/7,13/12$).  Based on the analysis in Section \ref{sec4}, except {\bf H-1} and {\bf H+1}, the model parameter $\alpha$ for all sub-Logitron model is in the category $\{ \alpha \in \R_{++} \;|\;  \frac{1}{1-\alpha}  = k \in \mathbb{Z} \setminus \{ 0 \} \}$. We summarize the parameter space of each sub-Logitron model in Table \ref{tablePar}. Four-fold cross validation~\cite{delgado14} is used to select the optimal parameters of nine sub-Logitron models and three models of LIBLINEAR. Due to the independency of each cross-validation process, it is easy to be implemented in parallel processing machines.

In terms of benchmark dataset, we use the well-organized datasets in~\cite{delgado14} while reporting the performance of the nine sub-Logitron models. In fact, they are pre-processed and normalized in each feature dimension with mean zero and variance one. The raw data are mostly in UCI machine learning repository. Note that, as commented in~\cite{wainberg16}, we reorganize the dataset in \cite{delgado14}. First, each dataset is separated into the training and testing data set which are not overlapped. Each training data set is randomly shuffled for $4$-fold cross validation. Among the dataset in~\cite{delgado14}, we use $118$ datasets after removing ambiguous dataset in terms of data splitting strategy. In Appendix, we list up all information of datasets such as number of instances, number of train data, number of test data, feature dimension, and number of classes. See Table~\ref{tableFullData} for more details. Last but not least, for multi-class datasets, we exploit the one-vs-all strategy, the most commonly used in multi-class classification based on a binary classifier. This strategy is also used in LIBLINEAR~\cite{fan08}. 

\begin{table*}
\centerline{
\begin{footnotesize}
\begin{tabular}{c|c||c|c|c|c||c|c|c||c|c|||c|c|c}\hline
{\bf Category}&{\bf Method}&\textbf{H-1}&\textbf{H-2}&\textbf{H-3}&\textbf{H-4}&\textbf{H+1}&\textbf{H+2}&\textbf{H+3}&\textbf{L-}&\textbf{L+}&\textbf{Logisitc}&\textbf{SVM}&\textbf{L2SVM}\\\hline\hline
\textbf{2-class}& acc(\%) &83.49&83.37&83.57&{\bf 83.60}&83.20&82.95&83.34&83.46&83.43&83.48&82.60&82.82\\\hline
                        & ranking      &6.13&{\bf 6.00}&6.05&6.19&7.00&7.35&6.41&6.09&6.05&6.21&7.49&7.02\\\hline
\textbf{M-class}& acc(\%) & 73.26&73.39&73.55&{\bf 73.59}&73.18&72.73&73.19&73.57&73.31&72.50&71.53&72.34\\\hline
                         & ranking      &5.87&6.32&5.70&5.71&6.77&7.70&6.75&{\bf 5.13}&6.25&6.61& 7.93&7.27\\\hline
\textbf{All-class}& acc(\%) & 77.33&77.36&77.54&{\bf 77.58}&77.17&76.80&77.23&77.51&77.34&76.87&75.94&76.51\\\hline
                          &  ranking   &5.97&6.19&5.84&5.90&6.86&7.56&6.61&{\bf 5.51}&6.17&6.45& 7.75&7.17\\\hline
\end{tabular}
\end{footnotesize}
}
\caption{Comparison of the performance of the various Logitron sub-models with the well-known linear classifiers in LIBLINEAR~\cite{fan08}. Overall, {\bf H-4} (i.e. the fourth-order SVM with an additional fourth root stabilization function) shows the best performance in terms of the classification accuracy for all-class problems. What is interesting on the Logitron model is that the cheapest classification model {\bf H+2} shows reasonable performance. Actually, it is comparable to the logistic regression in LIBLINEAR in terms of the classification accuracy. Interestingly, {\bf H+3} (the higher-order version of {\bf H+2}) shows better performance than {\bf H+2}. With respect to the Friedman ranking, L-Logitron (i.e., {\bf L-}) shows the best performance.
}\label{tableResult}
\end{table*}

The whole experiments are run five times and the averaged test score of each dataset is reported in Table \ref{tableTwo} and Table \ref{tableMc} in Appendix. In each experiment, the best parameters are chosen through the $4$-fold cross-validation. With the chosen best parameters, we minimize \eqref{linmin} with the whole training data in Table \ref{tableFullData} to find the hyperplane, i.e., $(w,b)$. Then we evaluate the performance of each classification model with test dataset in Table \ref{tableFullData}. For more details on CV-based minimization, see \cite{chang11}. All numerical results are summarized in Table \ref{tableResult}.  In terms of classification accuracy, H-Logitron {\bf H-4} is the best classification model and L-Logitron {\bf L-} obtains the best Friedman ranking~\cite{delgado14}. The H-Logitron submodels ({\bf H-2}, {\bf H-3}, and {\bf H-4}) are $k$-th order SVMs ($k=2,3,4$) with the corresponding $k$-th root stabilization functions. In this category, as we increase the order of the model, the performance is getting better. What is interesting is that {\bf H-2} (the second order SVM with root stabilization function) outperforms the classic second order SVM, i.e., L2SVM~\cite{fan08}. A H+Logitron subbmodel, i.e., the cheapest classification model {\bf H+2} with \eqref{lowcomplexityLoss}, also shows comparable performance to the classic logistic regression. 
\begin{figure}[t]
\centering
\includegraphics[width=3.5in]{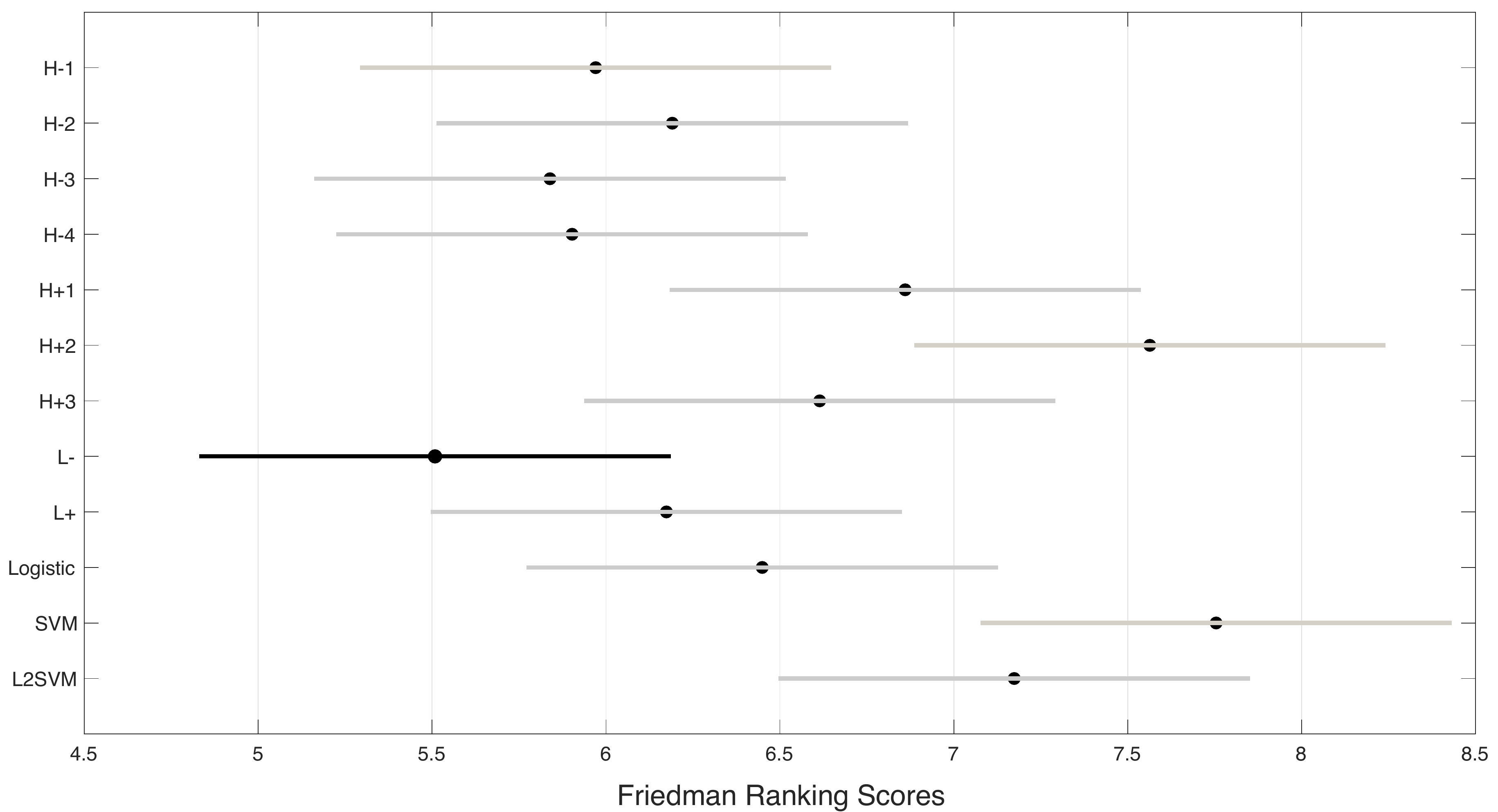}
\caption{Friedman Ranking Scores for all-class dataset. The L-Logitron submodel {\bf L-} shows the best performance.  Here we use $90\%$ confidence level.}
\label{fig:friedmanall}
\end{figure}
\begin{table*}
\centerline{
\begin{scriptsize}\begin{tabular}{c|c|c|c|c|c|c|c|c}
\hline
&\textbf{H-1}&\textbf{racc}&\textbf{H-2}&\textbf{racc}&\textbf{H-3}&\textbf{racc}&\textbf{H-4}&\textbf{racc}\\\hline\hline
1 &hill-valley  &10.78&hill-valley  &9.00&hill-valley  &12.10&hill-valley  &13.42\\\hline
2&acute-inflammation  &0.00&heart-hungarian  &0.20&heart-hungarian  &0.07&acute-inflammation  &0.00\\\hline
3&acute-nephritis  &0.00&acute-inflammation  &0.00&credit-approval  &0.01&acute-nephritis  &0.00\\\hline
4&waveform  &-0.25&acute-nephritis  &0.00&acute-inflammation  &0.00&heart-hungarian  &-0.07\\\hline
5&ozone  &-0.32&heart-cleveland  &-0.03&acute-nephritis  &0.00&dermatology  &-0.13\\\hline
6&twonorm  &-0.38&waveform  &-0.26&waveform  &-0.16&breast-cancer-wisc-diag  &-0.24\\\hline
7&credit-approval  &-0.46&ozone  &-0.37&dermatology  &-0.24&credit-approval  &-0.28\\\hline
8&heart-hungarian  &-0.48&breast-cancer-wisc-diag  &-0.45&ozone  &-0.29&waveform  &-0.28\\\hline
9&spectf  &-0.52&twonorm  &-0.46&breast-cancer-wisc-diag  &-0.38&ozone  &-0.32\\\hline
10&breast-cancer-wisc-diag  &-0.66&spectf  &-0.52&twonorm  &-0.51&spectf  &-0.52\\\hline
11&breast-cancer-wisc  &-0.78&breast-cancer-wisc  &-0.84&spectf  &-0.52&twonorm  &-0.52\\\hline
12&statlog-heart  &-0.79&dermatology  &-1.00&breast-cancer-wisc  &-0.61&breast-cancer-wisc  &-0.61\\\hline
13&dermatology  &-0.90&-&-&fertility  &-0.80&-&-\\\hline\hline
&\textbf{H+1}&\textbf{racc}&\textbf{H+2}&\textbf{racc}&\textbf{H+3}&\textbf{racc}&\textbf{L-}&\textbf{racc}\\\hline\hline
1&hill-valley  &11.05&hill-valley  &10.78&hill-valley  &10.91&hill-valley  &3.59\\\hline
2&heart-hungarian  &0.20&heart-hungarian  &0.61&credit-approval  &0.30&heart-hungarian  &0.07\\\hline
3&acute-inflammation  &0.00&acute-inflammation  &0.00&heart-hungarian  &0.07&acute-inflammation  &0.00\\\hline
4&acute-nephritis  &0.00&acute-nephritis  &0.00&acute-inflammation  &0.00&acute-nephritis  &0.00\\\hline
5&waveform  &-0.30&breast-cancer-wisc-diag  &-0.03&acute-nephritis  &0.00&breast-cancer-wisc-diag  &-0.10\\\hline
6&breast-cancer-wisc-diag  &-0.38&waveform  &-0.04&breast-cancer-wisc-diag  &-0.17&waveform  &-0.11\\\hline
7&ozone  &-0.38&ozone  &-0.27&waveform  &-0.28&credit-approval  &-0.28\\\hline
8&credit-approval  &-0.40&spectf  &-0.52&statlog-heart  &-0.35&ozone  &-0.30\\\hline
9&spectf  &-0.52&twonorm  &-0.62&ozone  &-0.48&twonorm  &-0.33\\\hline
10&twonorm  &-0.63&statlog-heart  &-0.64&twonorm  &-0.50&dermatology  &-0.35\\\hline
11&statlog-heart  &-0.79&dermatology  &-0.90&spectf  &-0.52&spectf  &-0.52\\\hline
12&vertebral-column-3clases  &-0.95&-&-&dermatology  &-0.68&statlog-heart  &-0.64\\\hline
13&dermatology  &-1.00&-&-&-&-&breast-cancer-wisc  &-0.90\\\hline\hline
&\textbf{L+}&\textbf{racc}&\textbf{Logistic}&\textbf{racc}&\textbf{SVM}&\textbf{racc}&\textbf{L2SVM}&\textbf{racc}\\\hline\hline
1&hill-valley  &8.93&hill-valley  &6.23&acute-inflammation  &0.00&acute-inflammation  &0.00\\\hline
2&heart-hungarian  &0.34&breast-cancer-wisc-diag  &0.32&acute-nephritis  &0.00&acute-nephritis  &0.00\\\hline
3&acute-inflammation  &0.00&heart-hungarian  &0.20&ozone  &-0.30&heart-hungarian  &-0.07\\\hline
4&acute-nephritis  &0.00&acute-inflammation  &0.00&breast-cancer-wisc-diag  &-0.31&ozone  &-0.27\\\hline
5&breast-cancer-wisc-diag  &-0.03&acute-nephritis  &0.00&echocardiogram  &-0.32&credit-approval  &-0.28\\\hline
6&statlog-heart  &-0.20&credit-approval  &-0.23&waveform  &-0.33&breast-cancer-wisc-diag  &-0.31\\\hline
7&waveform  &-0.29&ozone  &-0.24&twonorm  &-0.52&waveform  &-0.33\\\hline
8&twonorm  &-0.37&twonorm  &-0.33&credit-approval  &-0.86&statlog-heart  &-0.35\\\hline
9&ozone  &-0.38&waveform  &-0.48&mammographic  &-0.88&twonorm  &-0.49\\\hline
10&credit-approval  &-0.40&statlog-heart  &-0.50&-&-&breast-cancer-wisc  &-0.61\\\hline
11&spectf  &-0.52&dermatology  &-0.90&-&-&dermatology  &-0.79\\\hline
12&dermatology  &-1.00&breast-cancer-wisc  &-0.95&-&-&-&-\\\hline
\end{tabular}
\end{scriptsize}
}
\caption{Comparisons of the Best-$1\%$ sets of the nine Logitron sub-models (H-1, H-2, H-3, H-4, H+1, H+2, H+3, L-, and L+) and three models (Logistic regression, SVM, and L2SVM) in LIBLINEAR. Here racc means the relative classification accuracy against the virtual DWN classifier~\cite{delgado14}. Note that the virtual DWN classifier means the best classifier among $179$ classifiers, including boosting, neural network, and random forest, for each individual dataset with respect to the classification accuracy.
}\label{tableTop1}
\end{table*}

Table \ref{tableTop1} presents the dataset in the Best-$1\%$ set of each classifier in terms of the relative classification accuracy (racc$>-1$). Here, the relative classification accuracy (racc) is the subtraction of the accuracy of the virtual DWN in~\cite{delgado14} from the accuracy of each classifier. Note that the virtual DWN classifier means the best classifier among $179$ classifiers, including boosting, neural network, and random forest, for each individual dataset with respect to the classification accuracy. That is, it is not a specific classifier existed in the real world but an idealistic virtual classifier. Although the function space of the Logitron is linear, interestingly, the proposed Logitron model gets better performance than the optimal DWN classifier in some datasets such as 'hill-valley', 'acute-inflammation', 'acute-nephritis', 'heart-hungarian', 'credit-approval', etc. 
    
In Figure \ref{fig:reg}, \ref{fig:palpha}, and \ref{fig:pcx}, we summarize statistical information of the parameters $\lambda$, $\alpha$, and $c$ (or $c_{\alpha}$) which are selected via $4$-fold cross-validation with the training dataset in Table \ref{tableFullData}. Since we did the whole experiments five times, the histograms are generated with $590$ samples. They are normalized for probabilistic interpretation of the parameter data. For each model, we plot histograms of two datasets; Best-$1\%$(Left) and remainders (Right). Figure \ref{fig:reg} shows the normalized histogram of the $\lambda$ with respect to $\log_2(\lambda)$. The regularization parameter $\lambda$ of all Logitron sub-models for the Best-$1\%$ set are mainly located near $2^{-14}$ or $>2^0$. Note that Logitron is not convex with respect to $\lambda$, $w$, and $b$ at the same time. Thus, there are many local minima during the selection process of the regularization parameter with cross-validation. Due to the inherent ambiguity, we have many candidate for the best regularization parameter. Therefore, when the training accuracies are even, we simply select a regularization parameter having smaller value. As a result of the regularization parameter selection process, we have relatively high frequency at $2^{-15}$. 
Figure \ref{fig:palpha} visualizes $\alpha$ for various different Logitron sub-models. The Logitron with $\alpha<1$ (i.e., {\bf H-1} and {\bf L-}) in Best-$1\%$ prefers smaller value of $\alpha$ than the remainder set. Figure \ref{fig:pcx} demonstrates the preference of the margin parameter $c_{\alpha}$ in the Logitron submodels; {\bf H-2}, {\bf H-3}, {\bf H-4} and {\bf H+2}, {\bf H+3}. Overall, {\bf H-3},{\bf H-4}, and {\bf H+3} in Best-$1\%$ prefer $\abs{c_{\alpha}}=1$ to the remainder set. 

\begin{figure*}[t]
\centering
\includegraphics[width=5.5in]{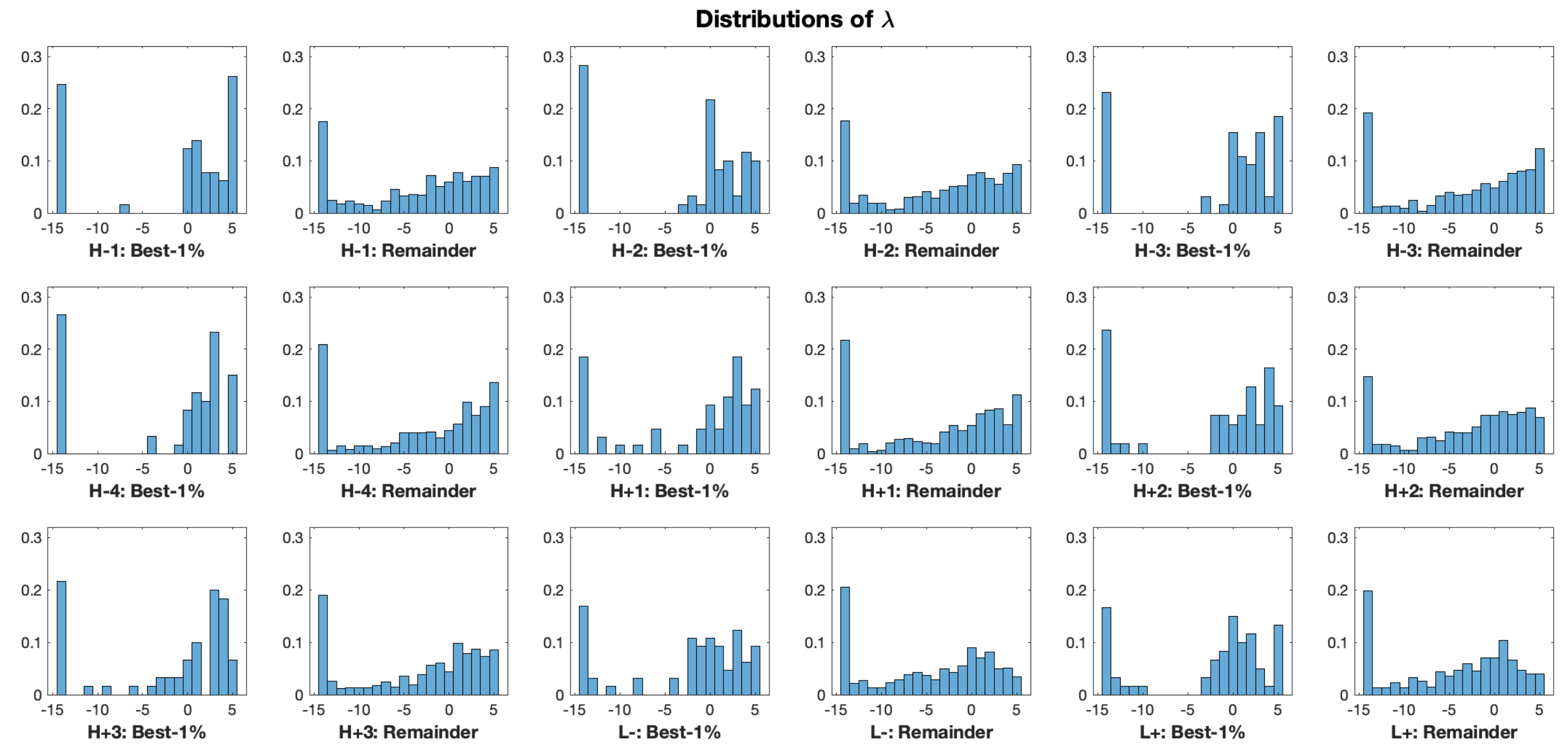}
\caption{A comparison of the regularization parameter $\lambda$ of all models. Note that the x-axis is $\log_2(\lambda)$ and the y-axis is the normalized frequencies of $\lambda$. Overall, the Logitron in the category of the Best-$1\%$ choose the regularization parameter near $2^{-14}$ or $>0$. } 
\label{fig:reg}
\end{figure*}

\begin{figure*}[t]
\centering
\includegraphics[width=5in]{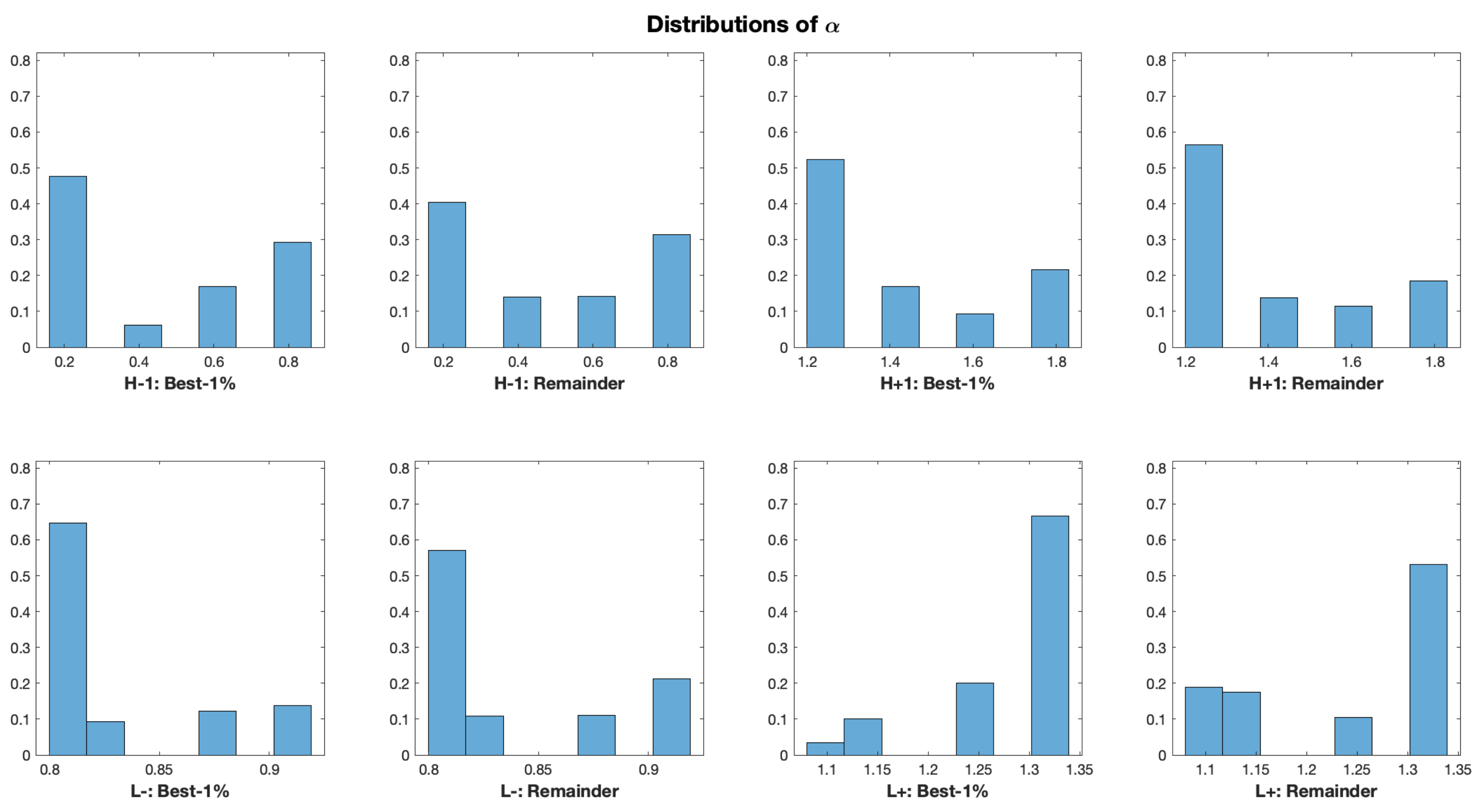}
\caption{A comparison of the model parameter $\alpha$ of the Logitron. The Logitron model with $\alpha<1$ (i.e. {\bf H-1} and {\bf L-}) in the category of Best-1\% prefer smaller value of $\alpha$.}
\label{fig:palpha}
\end{figure*}

\begin{figure*}[t]
\centering
\includegraphics[width=5in]{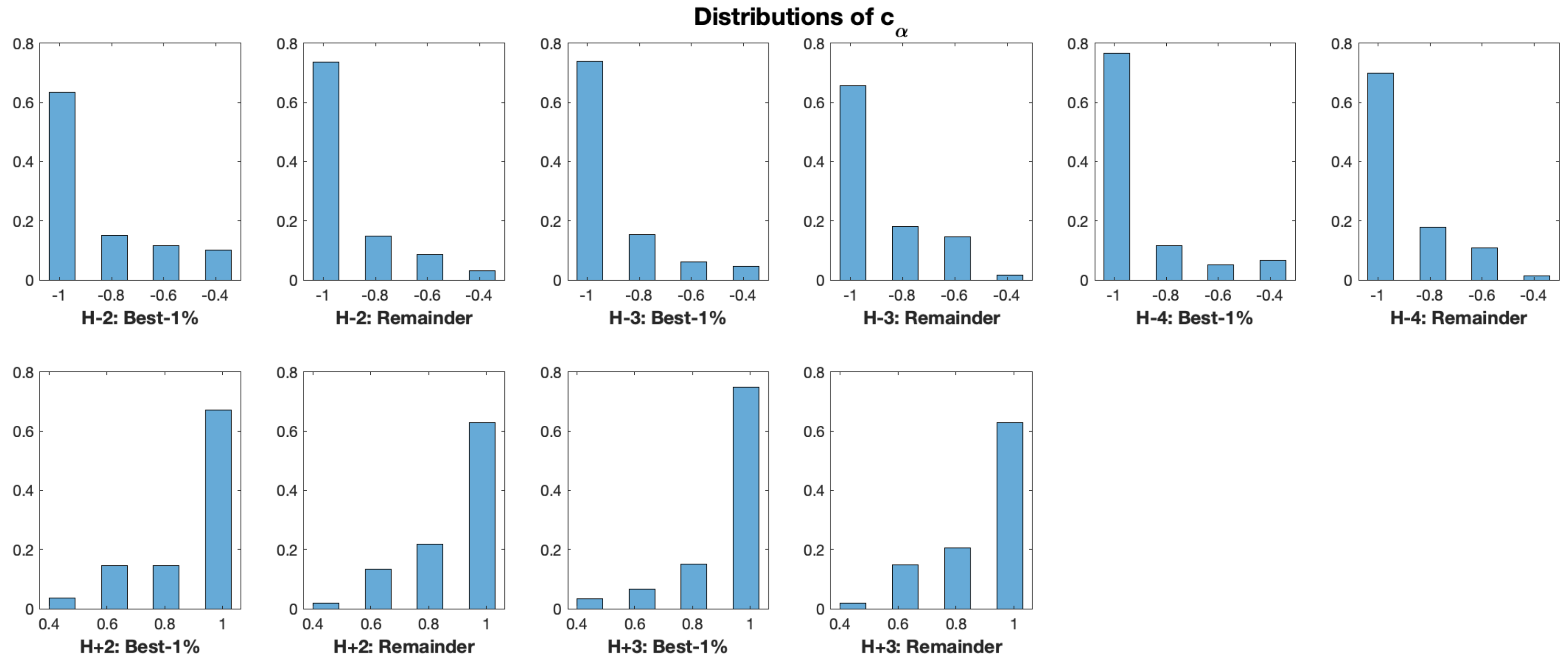}
\caption{A comparison of the margin parameter $c_{\alpha}$ of the Logitron model. In the Best-$1\%$ category, {\bf H-3}, {\bf H-4} and {\bf H+3} prefer $\abs{c_{\alpha}}=1$.} 
\label{fig:pcx}
\end{figure*}

\section{Conclusion\label{sec6}}
In this article, we have introduced a general convex classification framework, i.e., Logitron, which is an extended logistic loss function with the classic Perceptron loss function. The proposed Logitron has several useful features. A typical one is that it is differentiable on the whole real line for all $\alpha \in \R_{++}$. Therefore, it is easy to use the conventional optimization algorithm. Depending on the choice of the parameters, we have two different categories of models; the Hinge-Logitron model ($\abs{c_{\alpha}}=1$) and the Logistic-Logitron model ($\abs{c}=1$). A Hinge-Logitron model {\bf H-4} (the fourth-order SVM with an additional fourth root function) outperforms the various other sub-Logitron models and the models in LIBLINEAR~\cite{fan08} in terms of classification accuracy. Additionally, a simple classification model {\bf H+2} shows reasonable performance compared to the classic logistic regression. A  Logistic-Logitron model {\bf L-} shows the best performance in terms of Friedman ranking.  

\section*{Acknowledgments}
This article is supported by the Basic Science Program through the NRF of Korea funded by the Ministry of Education (NRF-2015R101A1A01061261). The Logitron is designed based on {\it the machine learning MATLAB package} which is available in \url{https://www.cs.ubc.ca/~schmidtm/Software/minFunc.html}.

\appendix
Table \ref{tableFullData} shows the various information of all datasets used in this work. Note that all datasets in Table \ref{tableFullData} are a corrected version of the datasets in \cite{delgado14} based on \cite{wainberg16}.  
Table \ref{tableTwo} and Table \ref{tableMc} report the numerical performance of various Logitron submodels and models in LIBLINEAR~\cite{fan08}. Here the DWN means the best result among $178$ classifiers in \cite{delgado14} for each dataset. A H-Logitron submodel {\bf H-4}, the fourth-order SVM with fourth root stabilization function, outperforms other classification models, including logistic regression, SVM, and L2SVM in LIBLINEAR~\cite{fan08}. 

\begin{table*}
\centerline{
\begin{scriptsize}\begin{tabular}{l||c|c|c|c|c||l||c|c|c|c|c}
\hline
\textbf{Dataset} &\textbf{\#Instance}&\textbf{\#Train}&\textbf{\#Test}&\textbf{Dim}&\textbf{\#Class}  & \textbf{Dataset} &\textbf{\#Instance}&\textbf{\#Train}&\textbf{\#Test}&\textbf{Dim}&\textbf{\#Class}  \\\hline\hline
\textbf{abalone  }&4177&2089&2088&8&3 & \textbf{acute-inflammation  }&120&60&60&6&2\\\hline
\textbf{acute-nephritis  }&120&60&60&6&2 & \textbf{adult  }&48842&32561&16281&14&2\\\hline
\textbf{annealing  }&798&399&399&31&5   & \textbf{arrhythmia  }&452&226&226&262&13\\\hline
\textbf{audiology-std  }&196&171&25&59&18 & \textbf{balance-scale  }&625&313&312&4&3\\\hline
\textbf{balloons  }&16&8&8&4&2 & \textbf{bank  }&4521&2261&2260&16&2\\\hline
\textbf{blood  }&748&374&374&4&2 &  \textbf{breast-cancer  }&286&143&143&9&2\\\hline
\textbf{breast-cancer-wisc  }&699&350&349&9&2 & \textbf{breast-cancer-wisc-diag  }&569&285&284&30&2\\\hline
\textbf{breast-cancer-wisc-prog  }&198&99&99&33&2 & \textbf{breast-tissue  }&106&53&53&9&6\\\hline
\textbf{car  }&1728&864&864&6&4 & \textbf{cardiotocography-10clases  }&2126&1063&1063&21&10\\\hline
\textbf{cardiotocography-3clases  }&2126&1063&1063&21&3 & \textbf{chess-krvk  }&28056&14028&14028&6&18\\\hline
\textbf{chess-krvkp  }&3196&1598&1598&36&2 & \textbf{congressional-voting  }&435&218&217&16&2\\\hline
\textbf{conn-bench-sonar-mines-rocks  }&208&104&104&60&2  & \textbf{conn-bench-vowel-deterding  }&528&264&264&11&11\\\hline
\textbf{connect-4  }&67557&33779&33778&42&2 & \textbf{contrac  }&1473&737&736&9&3\\\hline
\textbf{credit-approval  }&690&345&345&15&2 & \textbf{cylinder-bands  }&512&256&256&35&2\\\hline
\textbf{dermatology  }&366&183&183&34&6 & \textbf{echocardiogram  }&131&66&65&10&2\\\hline
\textbf{ecoli  }&336&168&168&7&8 & \textbf{energy-y1  }&768&384&384&8&3\\\hline
\textbf{energy-y2  }&768&384&384&8&3 & \textbf{fertility  }&100&50&50&9&2\\\hline
\textbf{flags  }&194&97&97&28&8 & \textbf{glass  }&214&107&107&9&6\\\hline
\textbf{haberman-survival  }&306&153&153&3&2 & \textbf{hayes-roth  }&160&132&28&3&3\\\hline
\textbf{heart-cleveland  }&303&152&151&13&5 & \textbf{heart-hungarian  }&294&147&147&12&2\\\hline
\textbf{heart-switzerland  }&123&62&61&12&5 & \textbf{heart-va  }&200&100&100&12&5\\\hline
\textbf{hepatitis  }&155&78&77&19&2 & \textbf{hill-valley  }&606&303&303&100&2\\\hline
\textbf{horse-colic  }&368&300&68&25&2 & \textbf{ilpd-indian-liver  }&583&292&291&9&2\\\hline
\textbf{image-segmentation  }&2310&210&2100&18&7 & \textbf{ionosphere  }&351&176&175&33&2\\\hline
\textbf{iris  }&150&75&75&4&3 & \textbf{led-display  }&1000&500&500&7&10\\\hline
\textbf{lenses  }&24&12&12&4&3 & \textbf{letter  }&20000&10000&10000&16&26\\\hline
\textbf{libras  }&360&180&180&90&15 & \textbf{low-res-spect  }&531&266&265&100&9\\\hline
\textbf{lung-cancer  }&32&16&16&56&3 & \textbf{lymphography  }&148&74&74&18&4\\\hline
\textbf{magic  }&19020&9510&9510&10&2&
\textbf{miniboone  }&130064&65032&65032&50&2\\\hline
\textbf{molec-biol-promoter  }&106&53&53&57&2&
\textbf{mammographic  }&961&481&480&5&2\\\hline
\textbf{molec-biol-splice  }&3190&1595&1595&60&3&
\textbf{mushroom  }&8124&4062&4062&21&2\\\hline
\textbf{musk-1  }&476&238&238&166&2&
\textbf{musk-2  }&6598&3299&3299&166&2\\\hline
\textbf{nursery  }&12960&6480&6480&8&5&
\textbf{oocytes-merluccius-nucleus-4d  }&1022&511&511&41&2\\\hline
\textbf{oocytes-merluccius-states-2f  }&1022&511&511&25&3&
\textbf{oocytes-trisopterus-nucleus-2f  }&912&456&456&25&2\\\hline
\textbf{oocytes-trisopterus-states-5b  }&912&456&456&32&3&
\textbf{optical  }&5620&3823&1797&62&10\\\hline
\textbf{ozone  }&2536&1268&1268&72&2&
\textbf{page-blocks  }&5473&2737&2736&10&5\\\hline
\textbf{parkinsons  }&195&98&97&22&2&
\textbf{pendigits  }&10992&7494&3498&16&10\\\hline
\textbf{pima  }&768&384&384&8&2&
\textbf{pittsburg-bridges-MATERIAL  }&106&53&53&7&3\\\hline
\textbf{pittsburg-bridges-REL-L  }&103&52&51&7&3&
\textbf{pittsburg-bridges-SPAN  }&92&46&46&7&3\\\hline
\textbf{pittsburg-bridges-T-OR-D  }&102&51&51&7&2&
\textbf{pittsburg-bridges-TYPE  }&105&53&52&7&6\\\hline
\textbf{planning  }&182&91&91&12&2&
\textbf{plant-margin  }&1600&800&800&64&100\\\hline
\textbf{plant-shape  }&1600&800&800&64&100&
\textbf{plant-texture  }&1599&800&799&64&100\\\hline
\textbf{post-operative  }&90&45&45&8&3&
\textbf{primary-tumor  }&330&165&165&17&15\\\hline
\textbf{ringnorm  }&7400&3700&3700&20&2&
\textbf{seeds  }&210&105&105&7&3\\\hline
\textbf{semeion  }&1593&797&796&256&10&
\textbf{soybean  }&307&154&153&35&18\\\hline
\textbf{spambase  }&4601&2301&2300&57&2&
\textbf{spect  }&265&79&186&22&2\\\hline
\textbf{spectf  }&267&80&187&44&2&
\textbf{statlog-australian-credit  }&690&345&345&14&2\\\hline
\textbf{statlog-german-credit  }&1000&500&500&24&2&
\textbf{statlog-heart  }&270&135&135&13&2\\\hline
\textbf{statlog-image  }&2310&1155&1155&18&7&
\textbf{statlog-landsat  }&6435&4435&2000&36&6\\\hline
\textbf{statlog-shuttle  }&58000&43500&14500&9&7&
\textbf{statlog-vehicle  }&846&423&423&18&4\\\hline
\textbf{steel-plates  }&1941&971&970&27&7&
\textbf{synthetic-control  }&600&300&300&60&6\\\hline
\textbf{teaching  }&151&76&75&5&3&
\textbf{thyroid  }&7200&3772&3428&21&3\\\hline
\textbf{tic-tac-toe  }&958&479&479&9&2&
\textbf{titanic  }&2201&1101&1100&3&2\\\hline
\textbf{trains  }&10&5&5&29&2&
\textbf{twonorm  }&7400&3700&3700&20&2\\\hline
\textbf{vertebral-column-2clases  }&310&155&155&6&2&
\textbf{vertebral-column-3clases  }&310&155&155&6&3\\\hline
\textbf{wall-following  }&5456&2728&2728&24&4&
\textbf{waveform  }&5000&2500&2500&21&3\\\hline
\textbf{waveform-noise  }&5000&2500&2500&40&3&
\textbf{wine  }&178&89&89&13&3\\\hline
\textbf{wine-quality-red  }&1599&800&799&11&6&
\textbf{wine-quality-white  }&4898&2449&2449&11&7\\\hline
\textbf{yeast  }&1484&742&742&8&10&
\textbf{zoo  }&101&51&50&16&7\\\hline
\end{tabular}
\end{scriptsize}
}
\caption{The list of all datasets used in this article. This is a corrected version of dataset available in \cite{delgado14} based on~\cite{wainberg16}. The most dataset in this Table is available in UCI repository as raw formats.}\label{tableFullData}
\end{table*}

\begin{table*}
\centerline{
\begin{scriptsize}\begin{tabular}{l||c|c|c|c|c|c|c|c|c||c|c|c||c}
\hline
&\textbf{H-1}&\textbf{H-2}&\textbf{H-3}&\textbf{H-4}&\textbf{H+1}&\textbf{H+2}&\textbf{H+3}&\textbf{L-}&\textbf{L+}&\textbf{Logisitc}&\textbf{SVM}&\textbf{L2SVM}&\textbf{DWN}\\\hline\hline
\textbf{acute-inflammation  }&{\bf 100.00}&{\bf 100.00}&{\bf 100.00}&{\bf 100.00}&{\bf 100.00}&{\bf 100.00}&{\bf 100.00}&{\bf 100.00}&{\bf 100.00}&{\bf 100.00}&{\bf 100.00}&{\bf 100.00}&100.00\\\hline
\textbf{acute-nephritis  }&{\bf 100.00}&{\bf 100.00}&{\bf 100.00}&{\bf 100.00}&{\bf 100.00}&{\bf 100.00}&{\bf 100.00}&{\bf 100.00}&{\bf 100.00}&{\bf 100.00}&{\bf 100.00}&{\bf 100.00}&100.00\\\hline
\textbf{balloons  }&{\bf 87.50}&{\bf 87.50}&{\bf 87.50}&{\bf 87.50}&{\bf 87.50}&{\bf 87.50}&{\bf 87.50}&{\bf 87.50}&{\bf 87.50}&{\bf 87.50}&{\bf 87.50}&{\bf 87.50}&100.00\\\hline
\textbf{blood  }&76.31&76.42&76.15&76.10&{\bf 76.74}&76.42&{\bf 76.74}&76.10&{\bf 76.74}&76.20&76.20&75.67&80.30\\\hline
\textbf{breast-cancer  }&{\bf 72.45}&71.75&{\bf 72.45}&72.31&69.51&70.21&69.79&71.89&70.07&71.05&70.21&70.91&76.20\\\hline
\textbf{breast-cancer-wisc  }&96.62&96.56&{\bf 96.79}&{\bf 96.79}&96.05&95.42&95.70&96.50&95.93&96.45&96.33&{\bf 96.79}&97.40\\\hline
\textbf{breast-cancer-wisc-diag  }&97.54&97.75&97.82&97.96&97.82&98.17&98.03&98.10&98.17&{\bf 98.52}&97.89&97.89&98.20\\\hline
\textbf{breast-cancer-wisc-prog  }&79.39&{\bf 80.20}&78.99&79.80&{\bf 80.20}&79.39&79.80&79.19&79.19&76.57&75.76&77.78&82.80\\\hline
\textbf{chess-krvkp  }&{\bf 97.00}&96.92&96.70&96.72&96.57&95.94&96.20&96.81&96.51&96.65&96.80&96.65&99.60\\\hline
\textbf{congressional-voting  }&60.18&59.45&57.24&56.87&56.96&56.59&56.68&58.80&57.97&58.16&{\bf 60.37}&57.79&63.20\\\hline
\textbf{conn-bench-sonar-mines-rocks  }&77.31&76.73&76.54&75.96&76.92&76.35&76.92&76.73&77.31&77.12&{\bf 78.27}&74.81&90.40\\\hline
\textbf{credit-approval  }&87.94&87.30&88.41&88.12&88.00&87.07&{\bf 88.70}&88.12&88.00&88.17&87.54&88.12&88.40\\\hline
\textbf{cylinder-bands  }&74.22&74.37&73.52&{\bf 74.69}&72.81&73.36&73.12&73.75&73.20&73.67&{\bf 74.69}&74.14&81.20\\\hline
\textbf{echocardiogram  }&84.92&84.92&84.31&84.62&83.38&82.15&82.77&84.62&84.00&85.23&{\bf 87.38}&83.69&87.70\\\hline
\textbf{fertility  }&88.80&88.80&{\bf 89.20}&88.40&87.60&87.60&87.20&88.00&86.80&88.40&86.00&86.00&90.00\\\hline
\textbf{haberman-survival  }&72.94&73.33&73.07&73.07&72.42&72.55&72.68&73.07&72.81&{\bf 73.86}&73.59&{\bf 73.86}&77.10\\\hline
\textbf{heart-hungarian  }&86.12&86.80&86.67&86.53&86.80&{\bf 87.21}&86.67&86.67&86.94&86.80&85.03&86.53&86.60\\\hline
\textbf{hepatitis  }&75.84&74.29&76.10&76.10&78.18&77.66&78.44&77.40&78.70&{\bf 78.96}&73.51&75.06&89.70\\\hline
\textbf{hill-valley  }&85.08&83.30&86.40&{\bf 87.72}&85.35&85.08&85.21&77.89&83.23&80.53&64.09&68.71&74.30\\\hline
\textbf{horse-colic  }&66.18&66.18&66.18&66.18&66.18&66.18&66.18&66.18&66.18&66.18&{\bf 66.76}&66.18&91.20\\\hline
\textbf{ilpd-indian-liver  }&71.62&72.03&{\bf 73.13}&72.65&72.16&71.68&72.65&72.58&72.03&71.75&71.48&72.92&77.60\\\hline
\textbf{ionosphere  }&87.09&86.86&86.97&87.20&87.43&87.20&87.43&86.86&87.09&87.89&{\bf 88.46}&86.86&95.50\\\hline
\textbf{miniboone  }&91.03&{\bf 91.05}&91.04&91.03&90.99&90.90&91.00&91.02&91.03&90.35&90.32&87.66&93.80\\\hline
\textbf{molec-biol-promoter  }&74.72&76.60&77.74&78.11&78.11&{\bf 79.25}&77.36&77.36&78.11&77.74&75.47&75.09&93.30\\\hline
\textbf{mammographic  }&82.63&82.79&82.63&82.75&82.58&82.50&82.88&82.50&82.71&82.29&{\bf 83.12}&82.92&84.00\\\hline
\textbf{nursery  }&89.81&89.89&89.84&89.84&90.11&{\bf 90.28}&89.98&89.79&89.93&89.85&89.57&89.81&100.00\\\hline
\textbf{oocytes-merluccius-nucleus-4d  }&82.62&{\bf 82.94}&82.78&82.50&81.80&80.98&81.72&82.35&82.15&82.39&80.70&82.90&86.00\\\hline
\textbf{oocytes-merluccius-states-2f  }&91.51&91.39&91.47&91.47&91.15&91.55&91.19&91.51&91.15&91.19&91.82&{\bf 92.17}&94.00\\\hline
\textbf{oocytes-trisopterus-states-5b  }&92.98&{\bf 93.16}&93.07&92.68&92.19&92.24&92.50&92.46&93.07&92.85&92.85&{\bf 93.16}&95.10\\\hline
\textbf{optical  }&94.52&94.46&94.46&94.51&94.68&94.38&94.59&{\bf 94.81}&94.54&94.66&94.29&94.47&98.70\\\hline
\textbf{pendigits  }&89.19&89.60&89.22&89.17&89.62&{\bf 89.78}&89.39&89.67&89.65&89.69&89.38&89.71&97.80\\\hline
\textbf{pittsburg-bridges-MATERIAL  }&83.02&83.40&85.28&85.28&84.15&84.53&84.53&85.28&84.53&87.92&{\bf 88.68}&{\bf 88.68}&91.30\\\hline
\textbf{pittsburg-bridges-SPAN  }&63.48&66.96&66.52&66.96&60.43&61.30&64.35&66.96&63.48&75.65&75.22&{\bf 76.09}&80.40\\\hline
\textbf{plant-margin  }&76.25&70.87&74.40&75.37&73.00&71.65&75.08&{\bf 77.07}&75.97&69.25&58.40&64.82&87.20\\\hline
\textbf{plant-texture  }&{\bf 79.00}&77.82&78.50&78.70&74.34&72.09&75.59&78.87&77.80&75.57&70.36&74.69&86.60\\\hline
\textbf{soybean  }&86.41&87.84&87.19&86.41&83.92&83.27&84.58&86.27&85.10&85.75&{\bf 88.10}&86.80&92.50\\\hline
\textbf{statlog-australian-credit  }&{\bf 67.71}&67.25&67.19&66.84&67.30&67.36&67.19&67.07&67.54&67.19&{\bf 67.71}&66.96&69.10\\\hline
\textbf{statlog-german-credit  }&75.68&75.68&75.76&75.52&77.24&76.56&{\bf 77.28}&76.12&76.92&77.04&75.64&77.16&79.00\\\hline
\textbf{statlog-heart  }&87.41&85.93&86.96&86.96&87.41&87.56&87.85&87.56&{\bf 88.00}&87.70&86.37&87.85&88.20\\\hline
\textbf{statlog-image  }&92.61&92.74&{\bf 92.80}&92.62&92.43&92.43&92.69&92.57&92.07&91.76&91.06&91.24&98.60\\\hline
\textbf{statlog-landsat  }&81.62&81.23&81.48&81.59&83.03&{\bf 83.24}&82.31&81.64&82.23&81.63&80.03&81.00&91.90\\\hline
\textbf{statlog-shuttle  }&94.48&93.48&93.45&93.45&95.31&{\bf 95.33}&94.81&93.54&94.37&93.12&92.07&92.50&100.00\\\hline
\textbf{twonorm  }&97.62&97.54&97.49&97.48&97.37&97.38&97.50&{\bf 97.67}&97.63&{\bf 97.67}&97.48&97.51&98.00\\\hline
\textbf{vertebral-column-2clases  }&82.97&{\bf 84.13}&83.61&83.61&81.81&81.03&82.97&82.97&83.23&82.06&82.06&83.23&87.40\\\hline
\textbf{vertebral-column-3clases  }&84.26&84.39&84.90&85.16&{\bf 86.45}&85.03&85.94&84.65&85.03&84.26&84.26&85.03&87.40\\\hline
\textbf{wall-following  }&70.75&69.17&69.06&69.13&69.52&67.05&68.68&69.23&69.57&69.50&{\bf 72.49}&66.41&99.90\\\hline
\textbf{waveform  }&86.85&86.84&86.94&86.82&86.80&{\bf 87.06}&86.82&86.99&86.81&86.62&86.77&86.77&87.10\\\hline\hline
\textbf{Mean}&83.49&83.37&83.57&{\bf 83.60}&83.20&82.95&83.34&83.46&83.43&83.48&82.60&82.82&89.25\\\hline
\end{tabular}
\end{scriptsize}
}
\caption{Comparison of various two class classification models. We report the averaged classification accuracy($\%$) of five times repeated experiments of each classifier with test data in Table \ref{tableFullData}. Here  DWN is not a specific classifier, but the best classifier among $178$ classifiers in \cite{delgado14} for each data set. Overall, a H-Logitron submodel {\bf H-4} shows the best performance.}\label{tableTwo}
\end{table*}

\begin{table*}
\centerline{
\begin{scriptsize}\begin{tabular}{l||c|c|c|c|c|c|c|c|c||c|c|c||c}
\hline
&\textbf{H-1}&\textbf{H-2}&\textbf{H-3}&\textbf{H-4}&\textbf{H+1}&\textbf{H+2}&\textbf{H+3}&\textbf{L-}&\textbf{L+}&\textbf{Logisitc}&\textbf{SVM}&\textbf{L2SVM}&\textbf{DWN}\\\hline\hline
\textbf{abalone  }&64.98&65.03&65.21&65.07&65.09&{\bf 65.51}&65.26&65.08&65.03&65.06&61.33&65.21&67.40\\\hline
\textbf{adult  }&84.27&84.16&84.25&84.27&84.21&84.09&84.16&84.27&84.29&84.29&{\bf 84.35}&84.06&86.20\\\hline
\textbf{annealing  }&86.82&{\bf 87.57}&86.97&86.72&86.47&86.52&86.17&86.72&86.57&86.97&87.22&87.22&99.00\\\hline
\textbf{arrhythmia  }&{\bf 69.03}&68.58&68.94&68.32&67.08&66.64&66.64&67.26&67.26&68.50&64.69&63.72&77.40\\\hline
\textbf{audiology-std  }&77.60&76.00&77.60&{\bf 80.00}&74.40&75.20&73.60&79.20&76.80&68.00&68.00&66.40&92.00\\\hline
\textbf{balance-scale  }&88.14&88.14&88.14&88.14&88.21&87.76&88.14&88.14&88.21&88.14&{\bf 88.40}&88.08&99.00\\\hline
\textbf{bank  }&{\bf 89.00}&88.81&88.94&88.97&88.88&88.68&88.95&88.81&88.87&88.84&88.49&88.82&90.50\\\hline
\textbf{breast-tissue  }&67.17&66.42&65.66&65.66&64.91&63.40&64.53&64.91&66.04&64.53&{\bf 69.06}&66.42&79.80\\\hline
\textbf{car  }&82.01&81.60&81.69&82.29&82.04&81.64&82.18&82.29&{\bf 82.41}&{\bf 82.41}&79.91&81.25&99.20\\\hline
\textbf{cardiotocography-10clases  }&76.56&76.39&77.06&77.06&78.55&78.40&{\bf 78.61}&77.74&78.48&77.93&75.52&77.31&88.50\\\hline
\textbf{cardiotocography-3clases  }&89.69&89.50&89.65&89.67&89.61&89.39&89.58&89.78&89.69&{\bf 89.93}&89.73&89.80&95.60\\\hline
\textbf{chess-krvk  }&27.94&27.58&27.89&{\bf 27.95}&27.55&26.56&27.03&27.94&27.73&27.84&14.73&27.62&88.80\\\hline
\textbf{conn-bench-vowel-deterding  }&53.11&{\bf 55.98}&54.85&53.18&52.88&53.94&52.88&52.20&53.03&51.74&51.21&54.09&100.00\\\hline
\textbf{connect-4  }&75.43&75.38&75.38&75.42&75.45&75.43&75.45&75.44&75.45&{\bf 75.47}&75.38&75.41&90.40\\\hline
\textbf{contrac  }&50.11&50.11&50.00&50.30&50.90&50.60&{\bf 51.11}&50.52&50.98&50.87&46.71&49.95&57.20\\\hline
\textbf{dermatology  }&97.70&97.60&98.36&{\bf 98.47}&97.60&97.70&97.92&98.25&97.60&97.70&96.61&97.81&98.60\\\hline
\textbf{ecoli  }&86.90&86.79&87.74&88.10&87.98&87.02&87.26&{\bf 88.45}&86.90&88.21&85.00&87.74&90.90\\\hline
\textbf{energy-y1  }&85.31&85.36&85.21&85.57&86.09&85.05&85.78&85.36&86.04&{\bf 87.14}&84.43&86.35&97.80\\\hline
\textbf{energy-y2  }&88.54&88.54&88.49&88.54&89.01&88.49&89.64&89.17&89.06&88.39&86.88&{\bf 90.05}&93.40\\\hline
\textbf{flags  }&54.85&54.23&54.64&53.81&53.40&53.61&53.61&53.81&52.37&52.78&{\bf 55.05}&54.23&70.10\\\hline
\textbf{glass  }&64.11&64.49&63.93&63.36&60.75&60.56&61.50&62.99&60.75&62.24&{\bf 66.17}&64.49&78.50\\\hline
\textbf{hayes-roth  }&{\bf 3.57}&{\bf 3.57}&{\bf 3.57}&{\bf 3.57}&{\bf 3.57}&{\bf 3.57}&{\bf 3.57}&{\bf 3.57}&{\bf 3.57}&{\bf 3.57}&{\bf 3.57}&{\bf 3.57}&92.90\\\hline
\textbf{heart-cleveland  }&62.38&{\bf 64.77}&62.38&61.46&60.53&60.53&60.66&60.40&60.26&61.19&61.72&62.65&64.80\\\hline
\textbf{heart-switzerland  }&38.69&37.38&37.70&38.69&40.98&40.00&40.33&38.69&{\bf 41.31}&36.07&37.05&36.07&53.20\\\hline
\textbf{heart-va  }&{\bf 32.20}&29.20&29.80&30.00&26.40&26.60&26.80&29.80&27.00&27.20&32.00&27.00&40.00\\\hline
\textbf{image-segmentation  }&24.15&22.43&25.70&26.23&27.94&{\bf 28.66}&28.38&26.48&27.83&25.09&20.37&26.37&85.00\\\hline
\textbf{iris  }&94.67&94.67&94.67&94.93&94.93&{\bf 96.00}&95.73&95.20&95.47&94.67&91.73&94.40&99.30\\\hline
\textbf{led-display  }&70.72&70.60&70.88&70.56&70.96&70.92&{\bf 71.08}&70.68&70.88&70.56&68.76&70.16&74.80\\\hline
\textbf{lenses  }&{\bf 75.00}&{\bf 75.00}&{\bf 75.00}&{\bf 75.00}&{\bf 75.00}&{\bf 75.00}&{\bf 75.00}&{\bf 75.00}&{\bf 75.00}&{\bf 75.00}&{\bf 75.00}&{\bf 75.00}&95.80\\\hline
\textbf{letter  }&71.83&68.89&70.71&71.43&{\bf 72.26}&70.81&72.03&72.21&72.24&72.20&59.47&70.24&97.40\\\hline
\textbf{libras  }&63.44&62.78&62.56&62.56&62.22&59.11&60.89&63.22&61.78&63.00&62.56&{\bf 64.33}&89.20\\\hline
\textbf{low-res-spect  }&87.32&87.85&87.70&{\bf 89.28}&86.79&86.42&87.09&88.08&86.87&86.87&86.87&89.06&93.40\\\hline
\textbf{lung-cancer  }&55.00&52.50&56.25&60.00&61.25&{\bf 62.50}&61.25&58.75&61.25&57.50&{\bf 62.50}&56.25&75.00\\\hline
\textbf{lymphography  }&80.81&83.51&83.24&83.51&83.24&{\bf 83.78}&83.24&{\bf 83.78}&82.70&{\bf 83.78}&80.27&82.43&89.20\\\hline
\textbf{magic  }&79.04&78.93&78.95&79.04&79.07&78.66&78.95&79.08&79.07&{\bf 79.10}&78.92&78.99&88.30\\\hline
\textbf{molec-biol-splice  }&82.48&82.38&82.51&82.60&82.75&{\bf 83.10}&82.81&82.60&82.66&82.47&82.21&82.07&96.30\\\hline
\textbf{mushroom  }&97.02&95.72&95.49&95.23&93.99&92.51&93.22&95.08&94.40&94.49&{\bf 97.66}&93.99&100.00\\\hline
\textbf{musk-1  }&83.61&84.12&83.53&82.94&83.45&80.76&84.20&84.03&{\bf 84.45}&83.11&84.37&83.36&93.70\\\hline
\textbf{musk-2  }&{\bf 95.13}&94.90&94.79&94.67&94.37&93.28&94.40&94.57&94.45&94.56&95.08&94.60&99.80\\\hline
\textbf{oocytes-trisopterus-nucleus-2f  }&79.96&{\bf 80.26}&79.87&79.25&78.51&77.28&78.38&78.99&78.86&78.55&78.60&79.04&86.80\\\hline
\textbf{ozone  }&97.08&97.03&97.11&97.08&97.02&97.13&96.92&97.10&97.02&{\bf 97.16}&97.10&97.13&97.40\\\hline
\textbf{page-blocks  }&96.20&96.20&96.25&96.21&96.19&{\bf 96.32}&96.20&96.25&96.29&96.27&95.99&96.02&97.50\\\hline
\textbf{parkinsons  }&{\bf 87.01}&{\bf 87.01}&86.80&84.95&82.89&78.56&81.86&84.74&82.89&82.06&83.51&84.54&94.40\\\hline
\textbf{pima  }&76.56&{\bf 77.14}&76.51&76.41&76.04&76.15&76.09&76.77&76.30&76.30&75.68&76.41&79.00\\\hline
\textbf{pittsburg-bridges-REL-L  }&69.41&67.06&69.80&69.80&68.24&65.88&68.24&{\bf 70.59}&67.84&68.63&70.20&67.84&79.80\\\hline
\textbf{pittsburg-bridges-T-OR-D  }&86.67&88.24&89.02&88.24&88.63&87.45&87.84&88.24&87.45&89.80&86.27&{\bf 90.20}&93.50\\\hline
\textbf{pittsburg-bridges-TYPE  }&63.08&63.85&64.62&62.69&53.85&51.92&56.54&58.85&56.54&{\bf 66.54}&58.08&63.85&76.00\\\hline
\textbf{planning  }&66.59&66.15&65.05&64.84&64.40&67.25&65.71&65.05&64.18&63.74&{\bf 71.43}&65.49&72.80\\\hline
\textbf{plant-shape  }&51.70&48.22&49.10&50.72&51.70&47.45&50.45&{\bf 53.32}&52.72&50.17&39.42&47.47&72.30\\\hline
\textbf{post-operative  }&64.89&{\bf 68.00}&67.11&66.67&64.44&63.11&66.67&66.67&66.22&55.56&57.78&55.56&74.20\\\hline
\textbf{primary-tumor  }&43.03&43.39&44.00&43.76&42.67&41.33&41.94&{\bf 44.24}&42.91&43.03&40.48&42.67&52.70\\\hline
\textbf{ringnorm  }&77.25&{\bf 77.40}&77.30&77.27&76.41&75.49&76.17&77.24&76.81&76.89&77.35&77.11&98.70\\\hline
\textbf{seeds  }&92.19&93.33&93.52&93.33&93.71&93.71&{\bf 94.10}&93.90&93.52&92.00&91.81&92.38&97.20\\\hline
\textbf{semeion  }&91.21&91.46&91.08&91.08&92.51&{\bf 92.81}&92.44&91.93&92.26&89.12&86.23&85.30&96.40\\\hline
\textbf{spambase  }&92.62&92.80&92.64&92.51&92.15&91.75&91.97&92.52&92.39&92.43&{\bf 92.86}&92.03&96.10\\\hline
\textbf{spect  }&60.86&63.87&64.62&64.95&65.27&65.27&65.16&63.44&65.38&66.34&{\bf 68.28}&65.27&72.20\\\hline
\textbf{spectf  }&{\bf 91.98}&{\bf 91.98}&{\bf 91.98}&{\bf 91.98}&{\bf 91.98}&{\bf 91.98}&{\bf 91.98}&{\bf 91.98}&{\bf 91.98}&56.47&57.11&57.86&92.50\\\hline
\textbf{statlog-vehicle  }&78.16&{\bf 78.49}&77.87&77.97&77.30&75.74&77.40&77.87&77.49&77.68&77.68&77.78&85.10\\\hline
\textbf{steel-plates  }&70.56&70.80&70.54&70.29&70.95&{\bf 71.18}&70.97&70.60&70.72&70.58&70.80&70.37&80.40\\\hline
\textbf{synthetic-control  }&92.60&91.53&92.80&92.73&93.67&93.00&{\bf 93.73}&92.93&{\bf 93.73}&92.20&87.93&89.40&99.70\\\hline
\textbf{teaching  }&{\bf 52.00}&47.73&46.67&48.27&46.13&46.67&45.87&49.07&45.87&47.47&51.20&45.33&64.20\\\hline
\textbf{thyroid  }&95.27&94.94&94.96&94.95&94.98&95.10&94.92&94.98&94.96&95.04&{\bf 95.47}&94.42&99.00\\\hline
\textbf{tic-tac-toe  }&{\bf 97.91}&{\bf 97.91}&{\bf 97.91}&{\bf 97.91}&{\bf 97.91}&{\bf 97.91}&{\bf 97.91}&{\bf 97.91}&{\bf 97.91}&{\bf 97.91}&{\bf 97.91}&{\bf 97.91}&100.00\\\hline
\textbf{titanic  }&{\bf 77.55}&{\bf 77.55}&{\bf 77.55}&{\bf 77.55}&{\bf 77.55}&{\bf 77.55}&{\bf 77.55}&{\bf 77.55}&{\bf 77.55}&{\bf 77.55}&77.44&{\bf 77.55}&79.10\\\hline
\textbf{trains  }&44.00&{\bf 60.00}&{\bf 60.00}&{\bf 60.00}&{\bf 60.00}&{\bf 60.00}&{\bf 60.00}&{\bf 60.00}&{\bf 60.00}&{\bf 60.00}&{\bf 60.00}&{\bf 60.00}&100.00\\\hline
\textbf{waveform-noise  }&{\bf 86.01}&85.98&85.98&{\bf 86.01}&85.75&85.61&85.70&86.00&85.86&85.89&85.62&85.95&87.40\\\hline
\textbf{wine  }&98.65&{\bf 98.88}&{\bf 98.88}&{\bf 98.88}&{\bf 98.88}&{\bf 98.88}&{\bf 98.88}&{\bf 98.88}&{\bf 98.88}&98.20&{\bf 98.88}&{\bf 98.88}&100.00\\\hline
\textbf{wine-quality-red  }&56.37&56.90&56.90&56.07&57.22&{\bf 57.35}&57.22&56.57&57.20&56.67&56.25&56.32&69.00\\\hline
\textbf{wine-quality-white  }&{\bf 53.70}&52.70&53.68&53.51&52.63&52.37&52.67&53.64&52.91&53.47&47.03&53.34&69.10\\\hline
\textbf{yeast  }&60.24&60.11&{\bf 60.38}&60.19&59.54&58.89&59.06&60.30&59.76&60.13&50.11&59.92&63.70\\\hline
\textbf{zoo  }&93.60&{\bf 96.00}&95.60&{\bf 96.00}&{\bf 96.00}&94.00&{\bf 96.00}&{\bf 96.00}&{\bf 96.00}&{\bf 96.00}&95.20&{\bf 96.00}&99.00\\\hline\hline
\textbf{Mean}&73.26&73.39&73.55&{\bf 73.59}&73.18&72.73&73.19&73.57&73.31&72.50&71.53&72.34&85.83\\\hline
\end{tabular}
\end{scriptsize}
}
\caption{Comparison of various multi-class classification models. We report the averaged classification accuracy($\%$) of five times repeated experiments of each classifier with test data in Table \ref{tableFullData}. Here  DWN is not a specific classifier, but the best classifier among $178$ classifiers in \cite{delgado14} for each data set. Overall, a H-Logitron submodel {\bf H-4} shows the best performance. 
}\label{tableMc}
\end{table*}

\end{document}